\documentclass[11pt]{article} 
\pdfoutput=1
\usepackage{enumerate}
\usepackage[OT1]{fontenc}
\usepackage[usenames]{color}
\usepackage{smile}
\usepackage[colorlinks,
            linkcolor=red,
            anchorcolor=blue,
            citecolor=blue
            ]{hyperref}
\usepackage{fullpage}
\usepackage[protrusion=true,expansion=true]{microtype}
\usepackage{pbox}
\usepackage{setspace}
\usepackage{tabularx}
\usepackage{float}
\usepackage{wrapfig}
\usepackage{enumitem}
\usepackage{microtype}
\usepackage{graphicx}
\usepackage{subfigure}
\usepackage{enumerate}
\usepackage{enumitem}
\usepackage{pgfplots}
\usepackage{booktabs} 

\usepackage{enumitem}
\makeatletter
\newcommand*{\rom}[1]{\expandafter\@slowromancap\romannumeral #1@}
\makeatother
\PassOptionsToPackage{subfigure}{tocloft}
\usepackage{xspace} 
\usepackage{hyperref}
\usepackage{url}
\usepackage{smile}
\usepackage{enumitem}
\usepackage{booktabs}       
\usepackage{wrapfig}
\usepackage{graphicx}
\usepackage{subfigure}
\newcolumntype{L}[1]{>{\raggedright\let\newline\\\arraybackslash\hspace{0pt}}m{#1}}
\newcolumntype{C}[1]{>{\centering\let\newline\\\arraybackslash\hspace{0pt}}m{#1}}
\title{\huge On the Interplay Between Misspecification and Sub-optimality Gap in Linear Contextual Bandits}

\author
{
    Weitong Zhang\thanks{Department of Computer Science, University of California, Los Angeles, CA 90095, USA; e-mail: {\tt wt.zhang@ucla.edu}} 
    ~and~
	Jiafan He\thanks{Department of Computer Science, University of California, Los Angeles, CA 90095, USA; e-mail: {\tt jiafanhe19@ucla.edu}} 
	~and~
	Zhiyuan Fan\thanks{IIIS, Tsinghua University, Beijing, China; e-mail: {\tt fan-zy19@mails.tsinghua.edu.cn}} 
	~and~
	Quanquan Gu\thanks{Department of Computer Science, University of California, Los Angeles, CA 90095, USA; e-mail: {\tt qgu@cs.ucla.edu}}
}

\begin{document}
\date{}
\maketitle

\begin{abstract}
We study linear contextual bandits in the misspecified setting, where the expected reward function can be approximated by a linear function class up to a bounded misspecification level $\zeta>0$. We propose an algorithm based on a novel data selection scheme, which only selects the contextual vectors with large uncertainty for online regression. We show that, when the misspecification level $\zeta$ is dominated by $\tilde \cO(\Delta / \sqrt{d})$ with $\Delta$ being the minimal sub-optimality gap and $d$ being the dimension of the contextual vectors, our algorithm enjoys the same gap-dependent regret bound $\tilde \cO ({d^2} /{\Delta})$ as in the well-specified setting up to logarithmic factors. In addition, we show that an existing algorithm SupLinUCB~\citep{chu2011contextual} can also achieve a gap-dependent constant regret bound without the knowledge of sub-optimality gap $\Delta$. Together with a lower bound adapted from~\citet{lattimore2020learning}, our result suggests an interplay between misspecification level and the sub-optimality gap: (1) the linear contextual bandit model is efficiently learnable when $\zeta \leq \tilde \cO({\Delta} / \sqrt{d})$; and (2) it is not efficiently learnable when $\zeta \geq \tilde \Omega({\Delta} / {\sqrt{d}})$. Experiments on both synthetic and real-world datasets corroborate our theoretical results.
\end{abstract}

\section{Introduction}
Linear contextual bandits \citep{li2010contextual,chu2011contextual,abbasi2011improved, agrawal2013thompson} have been extensively studied when the reward function can be represented as a linear function of the contextual vectors. 
However, such a well-specified linear model assumption sometimes does not hold in practice. This motivates the study of misspecified linear models. In particular, we only assume that the reward function can be approximated by a linear function up to some worst-case error $\zeta$ called \emph{misspecification level}. Existing algorithms for misspecified linear contextual bandits~\citep{lattimore2020learning, foster2020adapting} can only achieve an $\tilde \cO(d\sqrt{K} + \zeta K \sqrt{d} \log K)$ regret bound, where $K$ is the total number of rounds and $d$ is the dimension of the contextual vector. Such a regret, however, suggests that the performance of these algorithms will degenerate to be linear in $K$ when $K$ is sufficiently large. The reason for this performance degeneration is because existing algorithms, such as OFUL~\citep{abbasi2011improved} and linear Thompson sampling~\citep{agrawal2013thompson}, utilize all the collected data without selection. This makes these algorithms vulnerable to ``outliers'' caused by the misspecified model. Meanwhile, the aforementioned results do not consider the sub-optimality gap in the expected reward between the best arm and the second best arm. Intuitively speaking, if the sub-optimality gap is smaller than the misspecification level, there is no hope to obtain a sublinear regret. Therefore, it is sensible to take into account the sub-optimality gap in the misspecified setting, and pursue a gap-dependent regret bound.

The same misspecification issue also appears in reinforcement learning with linear function approximation, when a linear function cannot exactly represent the transition kernel or value function of the underlying MDP. In this case, \citet{du2019good} provided a negative result showing that if the misspecification level is larger than a certain threshold, any RL algorithm will suffer from an exponentially large sample complexity. This result was later revisited in the stochastic linear bandit setting by~\citet{lattimore2020learning}, which shows that a large misspecification error will make the bandit model not efficiently learnable. However, these results cannot well explain the tremendous success of deep reinforcement learning on various tasks~\citep{mnih2013playing, schulman2015trust, schulman2017proximal}, where the deep neural networks are used as function approximators with misspecification error.

In this paper, we aim to understand the role of model misspecification in linear contextual bandits through the lens of sub-optimality gap. By proposing a new algorithm with data selection, we can achieve a constant regret bound for such a problem. We also shows that the existing algorithm, SupLinUCB~\citep{chu2011contextual} can be also viewed as a boostrapped version of our proposed algorithm. Our contributions are highlighted as follows:
\begin{itemize}[leftmargin=*]
    \item We propose a new algorithm called DS-OFUL (Data Selection OFUL). DS-OFUL only learns from the data with large uncertainty. We prove an $\tilde \cO(d^2\Delta^{-1})$ constant gap-dependent regret\footnote{we use notation $\tilde \cO(\cdot)$ to hide the log factor other than number of rounds $K$} bound independent from $K$ when the misspecification level is small (i.e., $\zeta = \tilde \cO(\Delta / \sqrt{d})$) and the minimal sub-optimality gap $\Delta$ is known. Our regret bound even improves upon the gap-dependent regret in the well-specified setting~\citep{abbasi2011improved} from $\log(K)$ to constant regret bound. To the best of our knowledge, this is the first constant gap-dependant regret bound for misspecified linear contextual bandits as well as the well-specified linear bandit without any prior assumptions.
    \item We show that an existing algorithm, SupLinUCB~\citep{chu2011contextual}, can be viewed as a multi-level version of our proposed algorithm. With a fine-grained analysis, we are able to show that SupLinUCB can achieve $\tilde \cO(d^2\Delta^{-1})$ constant regret under the same condition of misspecification level without knowing the sub-optimality gap.
    \item We also prove a gap-dependent lower bound following the lower bound proof techniques in~\citet{du2019good, lattimore2020learning}. This, together with the upper bound, suggests an interplay between the misspecification level and the sub-optimality gap: the linear contextual bandit is efficiently learnable if $\zeta \le \tilde \cO(\Delta /\sqrt{d})$ while it is not efficiently learnable if $\zeta \ge \tilde \Omega(\Delta /\sqrt{d})$.
    \item Finally, we conduct experiments on the linear contextual bandit with both synthetic and real datasets, and demonstrate the superior performance of DS-OFUL algorithm and the effectiveness of SupLinUCB. This corroborates our theoretical results.
\end{itemize}

\noindent\textbf{Notation.} Scalars and constants are denoted by lower and upper case letters, respectively. Vectors are denoted by lower case boldface letters $\xb$, and matrices by upper case boldface letters $\Ab$. We denote by $[k]$ the set $\{1, 2, \cdots, k\}$ for positive integers $k$. For two non-negative sequence $\{a_n\}, \{b_n\}$, $a_n = \cO(b_n)$ means that there exists a positive constant $C$ such that $a_n \le Cb_n$, and we use $\tilde \cO(\cdot)$ to hide the $\log$ factor in $\cO(\cdot)$ other than number of rounds $T$ or episode $K$; $a_n = \Omega(b_n)$ means that there exists a positive constant $C$ such that $a_n \ge Cb_n$, and we use $\tilde \Omega(\cdot)$ to hide the $\log$ factor. For a vector $\xb  \in \RR^d$ and a positive semi-definite matrix $\Ab \in \RR^{d \times d}$, we define $\|\xb\|_{\Ab}^2=\xb^\top \Ab\xb$. For any set $\cC$, we use $|\cC|$ to denote its cardinality.

\section{Related Work}
In this section, we review the related work for misspecified linear bandits and misspecified reinforcement learning.

\noindent\textbf{Linear Contextual Bandits.}
There is a large body of literature on linear contextual bandits. For example, \citet{auer2002using,chu2011contextual, agrawal2013thompson} studied linear contextual bandits when the number of arms is finite. \citet{abbasi2011improved} proposed an algorithm called OFUL to deal with the infinite arm set. All these works come with an $\tilde \cO(\sqrt{K})$ problem-independent regret bound, and an $\cO(d^2\Delta^{-1}\log(K))$ gap-dependent regret bound is also given by~\citet{abbasi2011improved}.

\noindent\textbf{Misspecified Linear Bandits.}  \citet{ghosh2017misspecified} is probably the first work considering the misspecified linear bandits, which shows that the OFUL~\citep{abbasi2011improved} algorithm cannot achieve a sublinear regret in the presence of misspecification. They, therefore, proposed a new algorithm with a hypothesis testing module for linearity to determine whether to use OFUL~\citep{abbasi2011improved} or the multi-armed UCB algorithm. Their algorithm enjoys the same performance guarantee as OFUL in the well-specified setting and can avoid the linear regret under certain misspecification setting. \citet{lattimore2020learning} proposed a phase-elimination algorithm for misspecified stochastic linear bandits, which achieves an $\tilde \cO(\sqrt{dK} + \zeta K\sqrt{d})$ regret bound. For contextual linear bandits, both~\citet{lattimore2020learning} and \citet{foster2020adapting} proved an $\tilde \cO(d\sqrt{K} + \zeta K\sqrt{d})$ regret bound under misspecification. \citet{takemura2021parameter} showed that SupLinUCB can achieve a similar regret bound without the knowledge of the misspecification level. \citet{van2019comments} proved a lower bound of sample complexity, which suggests when $\zeta\sqrt{d} \ge \sqrt{8\log|\cD|}$, any best arm identification algorithm will suffer a $\Omega(2^d)$ sample complexity, where $\cD$ is the decision set. When the reward is deterministic and does not contain noise, they provided an algorithm using $\tilde \cO(d)$ sample complexity to identify a $\Delta$-optimal arm when $\zeta \le \Delta /\sqrt{d}$. \citet{lattimore2020learning} also mentioned that if $\zeta\sqrt{d} \le \Delta$, there exists a best arm identification algorithm that only needs to pull $\tilde \cO(d)$ arms to find a $\Delta$-optimal arm with the knowledge of $\zeta$. Note that although the exponential sample complexity lower bound for best-arm identification can be translated into a regret lower bound in linear contextual bandits, the algorithms for best-arm identification and the corresponding upper bounds cannot be easily extended to linear contextual bandits. Besides these works on misspecification, \citet{he22corruptions} studied the linear contextual bandits with adversarial corruptions, where the reward for each round can be corrupted arbitrarily. They assumed that the summation of the corruption up to $K$ rounds is bounded by $C>0$ and proposed an algorithm achieving $\tilde \cO(d\sqrt{K} + dC)$ regret bound with the known $C$. Since the corruption level $C=K\zeta$ in the misspecification setting, their result directly implied an $\cO(d\sqrt{K} + dK\zeta)$ linear regret, which differs from the optimal guarantee with a extra $O(\sqrt{d})$ factor.
Besides these series of work, \citet{camilleri2021high} also studied the robustness of kernel bandits with misspecification.  

\section{Preliminaries of Linear Contextual Bandits}
We consider a linear contextual bandit problem. In round $k \in [K]$, the agent receives a decision set $\cD_k \subset \RR^d$ and selects an arm $\xb_k \in \cD_k$ then observes the reward $r_k = r(\xb_k) + \eps_k$, where $r(\cdot): \RR^d \mapsto [0, 1]$ is a deterministic expected reward function and $\eps_k$ is a zero-mean $R$-sub-Gaussian random noise. i.e., $\EE[\mathrm{e}^{\lambda \eps_k} | \xb_{1:k}, \eps_{1:k-1}] \le \exp({\lambda^2R^2}/2), \forall k \in [K], \lambda \in \RR$.

In this work, we assume that all contextual vector $\xb \in \cD_k$  satisfies $\|\xb\|_2 \le L$ and the reward function $r(\cdot): \RR^d\rightarrow [0,1]$ can be approximated by a linear function $r(\xb) = \xb^\top\btheta^* + \eta(\xb)$, where $\eta(\cdot): \RR^d \mapsto [-\zeta, \zeta]$ is an unknown misspecification error function.
 We further assume $\|\btheta^*\|_2 \le B$ and for simplicity, we assume $B , L \ge 1$. We denote the optimal reward at round $k$ as $r_k^* = \max_{\xb \in \cD_k}r(\xb)$ and the optimal arm $\xb_k^* = \argmax_{\xb \in \cD_k}r(\xb)$. Our goal is to minimize the regret defined by $\text{Regret}(K) := \sum_{k=1}^K r_k^* - r(\xb_k)$.

In this paper, we focus on the minimal sub-optimality gap condition.

\begin{definition}[Minimal sub-optimality gap]\label{def:bandit-gap}
For each $\xb \in \cD_k$, the sub-optimality gap $\Delta_k(\xb)$ is defined by $\Delta_k(\xb) := r_k^* - r(\xb)$ and the minimal sub-optimality gap $\Delta$ is defined by $
    \Delta := \textstyle{\min_{k \in [K], \xb \in \cD_k}}\{\Delta_k(\xb): \Delta_k(\xb) > 0\}.
$
\end{definition}
Then we further assume this minimal sub-optimality gap is strictly positive, i.e., $\Delta > 0$.
\section{Constant Regret Bound with Known Sub-Optimality Gap $\Delta$}
\subsection{Algorithm}
In this subsection, we propose our algorithm, DS-OFUL, in Algorithm~\ref{alg:main}. The algorithm runs for $K$ rounds. At each round, the algorithm first estimates the underlying parameter $\btheta^*$ by solving the following ridge regression problem in Line~\ref{ln:bandit-reg1}:
\begin{align*}
    \btheta_k = \textstyle{\argmin_{\btheta}} \textstyle{\sum_{i \in \cC_{k-1}}} \left(r_i - \xb_i^\top\btheta\right)^2 + \lambda \|\btheta\|_2^2,
\end{align*}
where $\cC_{k-1}$ is the index set of the selected contextual vectors for regression and is initialized as an empty set at the beginning. After receiving the contextual vectors set $\cD_k$, the algorithm selects an arm from the optimistic estimation powered by the Upper Confidence Bound (UCB) bonus in Line~\ref{ln:decision}. In line~\ref{ln:select}, the algorithm adds the index of current round into $\cC_k$ if the UCB bonus of the chosen arm $\xb_k$, denoted by $\|\xb_k\|_{\Ub_{k}^{-1}}$, is greater than the threshold $\Gamma$.
Intuitively speaking, since the UCB bonus reflects the uncertainty of the model about the given arm $\xb$, Line~\ref{ln:select} discards the data that brings little uncertainty ($\|\xb\|_{\Ub_{k}^{-1}}$) to the model. Finally, we denote the total number of selected data in Line~\ref{ln:select} by $|\cC_K|$. We will declare the choices of the parameter $\Gamma, \beta$ and $\lambda$ in the next section.
\begin{algorithm}[htbp]
\caption{Data Selection OFUL (DS-OFUL)}\label{alg:main}
\begin{algorithmic}[1]
\REQUIRE Threshold $\Gamma$, radius $\beta$ and regularizer $\lambda$
\STATE Initialize $\cC_0 = \emptyset, \Ub_0 = \lambda \Ib, \btheta_0 = \zero$
\FOR{$k=1,\ldots,K$}
\STATE Set $\Ub_k = \lambda \Ib + \sum_{i \in \cC_{k-1}}\xb_i\xb_i^\top$. 
\STATE Set $\btheta_k = \Ub_k^{-1}\sum_{i \in \cC_{k-1}}r_i\xb_i$\label{ln:bandit-reg1}.
\STATE Receive the decision set $\cD_k$. 
\STATE Select $\xb_k = \argmax_{\xb \in \cD_k} \big\{\xb^\top\btheta_k + \beta\|\xb\|_{\Ub_k^{-1}}\big\}$. \label{ln:decision}
\STATE Receive reward $r_k$
\STATE $ \textbf{  if  } \|\xb_k\|_{\Ub_k^{-1}} \ge \Gamma \textbf{ then } \cC_k = \cC_{k-1} \cup \{k\} \textbf{  else  } \cC_k = \cC_{k-1}$ \label{ln:select}
\ENDFOR
\end{algorithmic}
\end{algorithm}
\subsection{Regret Bound}
In this subsection, we provide the regret upper bound of Algorithm~\ref{alg:main} and the regret lower bound for learning the misspecified linear contextual bandit.
\begin{theorem}[Upper Bound]\label{thm:main}

For any $0 < \delta < 1$, let $\lambda = B^{-2}$ and $\Gamma = \Delta / (2\sqrt{d}\iota_1)$ where $\iota_1 = (24 + 18R)\log((72 + 54R)LB\sqrt{d}\Delta^{-1}) + \sqrt{8R^2\log(1 / \delta)}$. Set $\beta = 1 + 4\sqrt{d\iota_2} + R\sqrt{2d\iota_3}$ where $\iota_2 = \log(3LB \Gamma^{-1})$, $\iota_3 = \log((1 +16L^2B^2\Gamma^{-2}\iota_2)/\delta)$. If the misspecification level is bounded by $2\sqrt{d}\zeta \iota_1 \le \Delta$, then with probability at least $1 - \delta$, the cumulative regret of Algorithm~\ref{alg:main} is bounded by
\begin{align*}
     \text{Regret}(K) &\le \frac{32\beta\sqrt{2d^3\iota_2\log(1 + 16d\Gamma^{-2}\iota_2)}\iota_1}\Delta.
\end{align*}
\end{theorem}
\begin{remark}
Since $\beta = \tilde \cO(\sqrt{d})$, Theorem~\ref{thm:main} suggests an $\tilde \cO(d^2\Delta^{-1})$ constant regret bound independent of the total number of rounds $K$ when $\zeta \le \tilde \cO(\Delta / \sqrt{d})$, 
which improves the logarithmic regret $\tilde \cO(d^2\Delta^{-1}\log K)$ in~\citet{abbasi2011improved} to a constant regret\footnote{{When we say constant regret, we ignore the $\log(1/\delta)$ factor in the regret as we choose $\delta$ to be a constant.}}. Note that our constant regret bound relies on the knowledge of the minimal sub-optimality gap $\Delta$, while the OFUL algorithm in~\citet{abbasi2011improved} does not need prior knowledge about the minimal sub-optimality gap $\Delta$.
\end{remark}
\begin{remark}
Our \emph{high probability} constant regret bound does not violate the lower bound proved in~\citet{hao2020adaptive}, which says that certain diversity condition on the contexts is necessary to achieve an \emph{expected} constant regret bound~\citep{papini2021leveraging}. Here we only provide a high-probability constant regret bound. When extending this high probability constant regret bound to expected regret bound, we have
\begin{align*}
    \EE[\text{Regret}(K)] \le \tilde \cO(d^2\Delta^{-1}\log(1 / \delta))(1 - \delta) + \delta K,
\end{align*}
which depends on $K$.
To obtain a sub-linear expected regret, we can choose $\delta = 1 / K$, which yields a logarithmic regret $\tilde \cO(d^2\Delta^{-1}\log(K))$ and does not violate the lower bound in \citet{hao2020adaptive}.
\end{remark}

\begin{remark}
Notably, \citet{papini2021leveraging} can achieve a constant expected  regret bound under certain diversity condition, which requires the contexts of arms span the whole $\RR^d$ space. 
In contrast, our constant regret bound does not need such an assumption and is a high-probability constant regret bound.
\end{remark}

\subsection{Key Proof Techniques} \label{sec:key-technique}
Here we present the key proof techniques for achieving the constant regret with the knowledge of sub-optimality gap $\Delta$. The detailed proof is deferred to Appendix~\ref{app:proof0}.

\paragraph{Regret decomposition}
The total regret over all $K$ rounds can be decomposed as follows
\begin{align}
    \text{Regret}(K) = \sum_{k \in \cC_K} \big(r_k^* - r(\xb_k)\big) + \sum_{k \notin \cC_K} \big(r_k^* - r(\xb_k)\big). \label{eq:main1}
\end{align}

\paragraph{Finite samples collected in $\cC_k$} Since we only adding the contextual arm with large uncertainty (i.e., $\|\xb\|_{\Ub_k^{-1}} \ge \Gamma$) into the set $\cC_k$, we can bound the number of samples in $\cC_k$ as $\cC_k = \tilde \cO(d\Gamma^{-2})$ which is claimed in the following lemma.

\begin{lemma}\label{lm:finite1}
Given $0 < \Gamma \le 1$, set $\lambda = B^{-2}$. For any $k \in [K]$, $|\cC_k| \le 16d\Gamma^{-2}\log(3LB\Gamma^{-1})$.
\end{lemma}

Then the following lemma suggests that a finite regression set $\cC_k$ can lead to a small confidence set with misspecification.

\begin{lemma}\label{lm:decompose1}
    Let $\lambda = B^{-2}$. For all $\delta > 0$, with probability at least $1 - \delta$, for all $\xb \in \RR^d, k \in [K]$, the prediction error is bounded by:
    \begin{align*}
        |\xb^\top(\btheta_k - \btheta^*)| & \le \left(1 + R\sqrt{2d\iota} + \zeta\sqrt{|\cC_k|}\right)\|\xb\|_{\Ub_k^{-1}},
    \end{align*}
    where $\iota = \log((d + |\cC_k|L^2B^2) / (d\delta))$ and $|\cC_k|$ is the total number of data used in regression at the $k$-th round.
\end{lemma}

Comparing the confidence radius $\tilde \cO(R\sqrt{d} + \zeta\sqrt{|\cC_{k}|})$ here with the conventional radius $\tilde \cO(R\sqrt{d})$ in OFUL, one can find that the misspecification error will affect the radius by an $\sqrt{|\cC_K|}$ factor. If we use all the data to do regression, the confidence radius will be in the order of $\tilde \cO(\sqrt{K})$ and therefore will lead to a $\cO(K\sqrt{\log K})$ regret bound (see Lemma 11 in~\citet{abbasi2011improved}). This makes the regret bound vacuous. In contrast, in our algorithm, the confidence radius is only $\sqrt{|\cC_K|}$ where $|\cC_K|$ is finite given Lemma~\ref{lm:finite1}. As a result, our regret bound will not grow with $K$ as in OFUL and will be smaller.

\paragraph{Skipped rounds are optimal}

Given the fact that the selected arm set $\cC_k$ is finite, the rest of the proof is simply showing that the skipped rounds $k\notin \cC_k$ are optimal and will not incur regret. Since we have $\|\xb\|_{\Ub_k^{-1}} \le \Gamma$ for those skipped rounds, the sub-optimality is bounded by the following (informal) lemma.
\begin{lemma}\label{lm:mis-informal}
The instantaneous regret for round $k \notin \cC_k$ is bounded by
\begin{align*}
    \Delta_k(\xb_k)  \le 2\zeta + 2\beta\|\xb_k\|_{\Ub_k^{-1}} \le \tilde \Theta (\zeta + \Delta + \sqrt{d}\Gamma),
\end{align*}
Setting $\Gamma = \tilde \Theta (\Delta/\sqrt{d})$ suggests that the instantaneous regret $\Delta_k(\xb_k) \le \Delta$, which means no instantaneous regret occurs on round $k$.
\end{lemma}

\paragraph{Achieving the constant regret}
To wrap up, as~\eqref{eq:main1} suggests, for rounds $k \in \cC_K$, we can follow the gap-dependent regret analysis in~\citet{abbasi2011improved} and obtain an $\tilde \cO(d^2\log(|\cC_K|)/\Delta)$ gap-dependent regret bound, which is independent of $K$ according to Lemma~\ref{lm:finite1}. For rounds $k \notin \cC_K$, Lemma~\ref{lm:mis-informal} guarantees a zero instantaneous regret. Putting them together yields the claimed constant regret bound.

\section{Constant Regret Bound with Unknown Sub-Optimality Gap $\Delta$}

\subsection{Algorithm}
Although Algorithm~\ref{alg:main} can achieve a constant regret, it requires the knowledge of sub-optimality gap $\Delta$. To tackle this problem, we propose a new algorithm that does not require the knowledge of sub-optimality gap $\Delta$. 

The algorithm is described in Algorithm~\ref{alg:suplin}. It inherits the arm elimination method from SupLinUCB~\citep{chu2011contextual}. A similar algorithm is also presented for misspecified linear bandits in~\citet{takemura2021parameter}. 

Algorithm~\ref{alg:suplin} works as follows. At each round $k \in [K]$, the algorithm maintains $l$ levels of ridge regression with different set $\cC_{k-1}^l$, where the estimation error for the $l$-th level is about $\beta(l)2^{-l}$ (we will prove this in the latter analysis). Then starting from the first level $l=1$ and the received decision set $\cD_k$, if there exists an arm in the decision set with a large uncertainty (i.e., $\|\xb\|_{(\Ub_k^l)^{-1}} \ge 2^{-l}$), the algorithm directly selects that arm (Line~\ref{ln:large-uncertainty}). According to Lemma~\ref{lm:finite1} in the analysis of DS-OFUL, the number of selected contexts at each level should be bounded. If the uncertainty for all arms is smaller than the threshold $2^{-l}$, the algorithm follows the arm elimination rule, which reduces the decision set into 
\begin{align}
     \cD_k^{l+1} = \left\{\xb: \xb \in \cD_k^{l}, r_k^l(\xb_k^l) - r_k^l(\xb) \le 2\beta(l) 2^{-l}\right\}. \label{eq:arm-elimination}
\end{align}
Then the algorithm enters the next level $l + 1$ until it reaches $\log(k)$-th level as Line~\ref{ln:ends} suggests. For the level $l \ge \log(k)$, the algorithm directly selects the arm with highest optimistic reward on Line~\ref{ln:large-reward} and does not add the index $k$ to the regression set $\cC_k^l$ as on Line~\ref{ln:endif} since the uncertainty is small enough.


Algorithm~\ref{alg:suplin} can be viewed as the multi-level version of Algorithm~\ref{alg:main} boosted by the peeling technique. Algorithm~\ref{alg:suplin} does not require the knowledge of the sub-optimality gap $\Delta$: if $\Delta$ is known, one can directly jump to a specific level $l_{\Delta} = \tilde \cO(\log(d/ \Delta))$, where the prediction error is bounded by $2\beta(l_{\Delta})2^{-l_{\Delta}} = \tilde \cO(\Delta)$ and is sufficient to achieve zero-instantaneous regret. 
However, when the $\Delta$ is unknown, Algorithm~\ref{alg:suplin} has to do a grid search over $2^{-1}, 2^{-2}, \cdots 2^{-l_{\Delta}}, \cdots$ and waste some of the samples to learn the first $l_{\Delta} - 1$ levels. We will revisit and compare the difference between these two algorithms in the later regret analysis.
\begin{algorithm}[ht]
\caption{SupLinUCB}\label{alg:suplin}
\begin{algorithmic}[1]
\REQUIRE Regularization $\lambda$, confidence radius $\beta(\cdot)$
\STATE Initialize $\cC^l_0 = \emptyset$ for all $l \in [\lceil \log(K)\rceil]$
\FOR {$k = 1, 2, \cdots K$}
\STATE Set $\cD_k^1 = \cD_k$ and $l = 1$
\REPEAT
\STATE Set $\Ub_{k}^l = \lambda \Ib + \sum_{i \in \cC_{k-1}^l}\xb_i\xb_i^\top$
\STATE Set $\btheta_{k}^l = (\Ub_{k}^{l})^{-1}\sum_{i \in \cC_{k-1}^l}r_i\xb_i$
\STATE Set $r_k^l(\xb) = \xb^{\top} \btheta_{k}^l + \beta(l) \left \|\xb\right\|_{(\Ub_{k}^{l})^{-1}}$
\STATE Select action $\xb_k^l = \argmax_{\xb \in \cD_k^{l}}r_k^l(\xb)$
\IF {$\max_{\xb \in \cD_k^{l}}\left \| \xb \right\|_{(\Ub_{k}^{l})^{-1}}  \ge 2^{-l}$}
    \STATE Choose $\xb_k = \argmax_{\xb \in \cD_k^{l}} \left \| \xb \right\|_{(\Ub_{k}^{l})^{-1}}$ \label{ln:large-uncertainty}
    \STATE Update $\cC_k^l = \cC_{k-1}^l \cup \{k\}$ 
    \STATE Keep $\cC_k^{l'} = \cC_{k-1}^{l'}$ for all $l' \neq l$
\ELSIF {$k \le 4^{l}d$}\label{ln:ends}
    \STATE Choose $\xb_k = \xb_k^l$ \label{ln:large-reward}
    \STATE Keep $\cC_k^{l'} = \cC_{k-1}^{l'}$ for all $l'\geq 1$ \label{ln:endif}
\ELSE
    \STATE Set $\cD_k^{l+1}$ according to~\eqref{eq:arm-elimination}
    \STATE Increase $l = l+1$
\ENDIF
\UNTIL{$\xb_k$ is chosen}
\STATE Take action $\xb_k$ and receive reward $r_k$
\ENDFOR
\end{algorithmic}
\end{algorithm}
\subsection{Regret Bound}

This subsection provides the regret upper bound for Algorithm \ref{alg:suplin}.

\begin{theorem}[Upper Bound]\label{thm:main-sup}
For any $0 < \delta < 1$, let $\lambda = B^{-2}$. For every integer $l > 0$, set $\beta(l) = 1 + R\sqrt{2d\iota_2(l)}$ where $\iota_{2}(l) = \log((d2^{l} + 16 L^2 B^2 8^l\iota_{1}(l)) / (d \delta))$ and $\iota_{1}(l) = \log\left(3LB2^l\right)$. If the misspecification level is bounded by $4l_{\Delta}\zeta\left(1 + 4\sqrt{d\iota_{1}(l_{\Delta})}\right) < \Delta$ where $l_{\Delta}$ is the minimal solution to $l_{\Delta} > \log(8\beta(l_{\Delta})/ \Delta)$, then with probability at least $1 - \delta$, the cumulative regret of Algorithm~\ref{alg:main} is bounded by
\begin{align*}
    &\text{Regret}(K) \le \frac{2560 d \beta^2(l_{\Delta}) \iota_{1}(l_{\Delta})}{\Delta}.
\end{align*}
\end{theorem}

\begin{remark}
    Since $\beta(l) = \tilde \cO(\sqrt{dl})$ and $l_{\Delta} = \tilde \cO(\log(d/ \Delta))$, Theorem~\ref{thm:main-sup} suggests that SupLinUCB enjoys a constant regret bound $\tilde \cO(d^2 \Delta^{-1})$ when $\zeta \leq \tilde \cO(\Delta/\sqrt{d})$, which is independent of the total number of rounds $K$.
    Note that in Algorithm~\ref{alg:suplin}, the choices of $\lambda$ and $\beta_{l}$ do not depend on the sub-optimality gaps $\Delta$ and misspecification level $\zeta$. 
\end{remark}

\begin{remark}
When $\zeta \ge \Delta / \sqrt{d}$, it is hard to provide a gap-dependent regret bound due to the large misspecification level $\zeta$. However, a gap-independent regret bound of $\tilde \cO(\sqrt{dK} + \sqrt{d}\zeta K\log(K))$ is proved in~\citet{takemura2021parameter}, which suggests the performance of SupLinUCB algorithm will not significantly decrease when the condition on misspecification does not hold.     
\end{remark}

\begin{remark}\label{rm:3}
Comparing the constant factors of DS-OFUL (Algorithm~\ref{alg:main}) and SupLinUCB (Algorithm~\ref{alg:suplin}) on the dominating terms $\tilde \cO(\beta^2d / \Delta)$, one can find that the constant factors of SupLinUCB is significantly larger than DS-OFUL. This is because it takes more samples to learn the first $l_{\Delta} - 1$ levels in SupLinUCB while DS-OFUL directly learns the $l_{\Delta}$-th level. Therefore, despite having the same order of constant regret bound (in big-O notation), one can expect that SupLinUCB has a worse performance than DS-OFUL (when $\Delta$ is known or can be estimated by grid search).
\end{remark}
\subsection{Key Proof Techniques}
Here we provide additional proof techniques besides the techniques discussed in Section~\ref{sec:key-technique}. First of all, Lemmas~\ref{lm:finite1} and~\ref{lm:decompose1}, which are built on a single level selected by $\|\xb\|_{\Ub_k^{-1}} \ge \Gamma$, can be generalized to the following lemmas for all levels $l$. The detailed proof are deferred to Appendix~\ref{sec:bandit-sup}.
\begin{lemma}\label{lm:set1}
Set $\lambda = B^{-2}$, for any $k \in [K]$ and $l > 0$, $|\cC_k^l| \le 16d4^l\iota_1(l)$, where $\iota_1(l) = \log\left(3LB2^l\right)$.
\end{lemma}
\begin{lemma}\label{lm:concent1}
Set $\lambda = B^{-2}$. For any level $l > 0$, for any $\delta > 0$, with probability at least $1 - \delta$, for all $k \in [K]$, the prediction error is bounded by
\begin{align*}
    \left|\xb^\top (\btheta_{k}^l - \btheta^*)\right| \le \left(1 + R\sqrt{2d\iota_{2}(l)} + \zeta \sqrt{\left|\cC_{k}^l\right|}\right)\|\xb\|_{(\Ub_k^l)^{-1}},
\end{align*}
for all $\xb$ such that $\|\xb\|_2 \le L$, where $\iota_{2}(l) = \log((d + |\cC_{k}^l| L^2 B^2) / (d \delta))$.
\end{lemma}

The following two proof techniques are crucial to prove constant regret bound of Algorithm \ref{alg:suplin}.

\paragraph{Optimal arm is never eliminated}
Considering the optimal arm in the eliminated set, which is defined by $\xb_k^{l, *} = \argmax_{\xb \in \cD_l}r(\xb)$. Obviously $\xb_k^{1, *} = \xb_k^*$. The following (informal) lemma says that the decision set always contains a nearly optimal action $\xb_k^{l, *}$:
\begin{lemma}[informal]\label{lm:optimal1}
    For any level $l > 0$, assume some good events hold, then there exists $\xb_k^{l,*} \in \cD_k^l$, such that $r(\xb^*_k) - r(\xb_k^{l,*}) \leq 2(l-1)\zeta\left(1 + 4 \sqrt{d\iota_1(l)}\right)$ where $\iota_1(l) = \log\left(3LB2^l\right)$.
\end{lemma}

Given the result of Lemma~\ref{lm:optimal1} and the existence of the sub-optimality gap $\Delta$, we have $\xb_k^{l, *} = \xb_k^*$ when $l$ is not too large. This means that the optimal arm is never eliminated from the decision set $\cD^l$. 

\paragraph{Sub-optimal arms are all eliminated}
Intuitively speaking, at level $l$, the prediction error is bounded by $\tilde \cO(\beta(l)\cdot 2^{-l})$ with some additional misspecification term $\zeta$. Therefore, when we eliminate the arms at level $l$, the sub-optimality of the arms in $\cD^l$ is bounded by the following (informal) lemma:
\begin{lemma}[informal]\label{lm:decision-set1}
    For any level $l > 0$, for any arm $\xb \in \cD_k^l$, $r(\xb^*_k) - r(\xb) \leq 4\beta(l) 2^{-l} + 2l\zeta\left(1 + 4 \sqrt{d\iota_1(l)}\right)$ where $\iota_1(l) = \log\left(3LB2^l\right)$.
\end{lemma}
Given Lemma~\ref{lm:decision-set1}, we know that when $l$ is sufficiently large (e.g., larger than $l_{\Delta}$), all $\xb \in \cD_k^l$ enjoys a sub-optimality less than $\Delta$. Combining with the existence of sub-optimality gap $\Delta$, we know that all of the sub-optimal arms are eliminated after level $l_{\Delta}$.

\paragraph{Regret decomposition}
Given Lemma~\ref{lm:set1} and Lemma~\ref{lm:decision-set1}, the regret over all $K$ rounds can be decomposed into
\begin{align}
    \text{Regret}(K) &= \sum_{k=1}^K \left(r(\xb_k^*) - r(\xb_k)\right)  = \sum_{l\ge 1}\sum_{k \in \cC_K^l}\left(r(\xb_k^*) - r(\xb_k)\right) = \sum_{l = 1}^{l_{\Delta}}\sum_{k \in \cC_K^l}\left(r(\xb_k^*) - r(\xb_k)\right)\notag,
\end{align}

where the last equality is due to the fact that no regret occurs after $l > l_{\Delta}$. For each level $l \le l_{\Delta}$, the summation of the instantaneous regret within $k \in \cC_K^l$ can be bounded following the gap-dependent regret bound of~\citet{abbasi2011improved} to obtain a $\tilde \cO(d^2\log|\cC_K^l|/\Delta)$ regret bound which is independent from $K$. Then taking the summation over $l \le l_{\Delta}$ yields the claimed constant regret bound.

\section{Lower Bound}
Following a similar idea in~\citet{lattimore2020learning}, we prove a gap-dependent lower bound for misspecified stochastic linear bandits. Note that stochastic linear bandit can be seen as a special case of linear contextual bandits with a fixed decision set $\cD_k=\cD$ across all round $k\in[K]$. 
Similar results and proof can be found in~\citet{du2019good} for episodic reinforcement learning. 

\begin{theorem}[Lower Bound]\label{thm:lower}
Given the dimension $d$ and the number of arms $|\cD|$, for any $\Delta \le 1$ and $\zeta \ge 3\Delta\sqrt{8\log(|\cD|)/(d - 1)}$, there exists a set of stochastic linear bandit problems $\bTheta$ with minimal sub-optimality gap $\Delta$ and misspecification error level $\zeta$, such that for any algorithm that has a sublinear expected regret bound for all $\btheta \in \bTheta$, i.e., $\EE[\text{Regret}_{\btheta}(K)] \le CK^\alpha$ with $C > 0$ and $ 0 \le \alpha < 1$, we have 
\begin{itemize}[leftmargin = *]
    \item When $K \leq \cO(|\cD|)$, the expected regret is lower bounded by $\EE_{\btheta \sim \text{Unif.}(\bTheta)}[\text{Regret}_{\btheta}(K)] \ge K\Delta$.
    \item When $K \geq \Omega(|\cD|)$, the expected regret is lower bounded by $\sup_{\btheta \in \bTheta}\EE[\text{Regret}_{\btheta}(K)] \ge \tilde \Omega(|\cD|\log(K)\Delta^{-1})$.
\end{itemize}
\end{theorem}

\begin{remark}\label{rm:1}
Theorem~\ref{thm:lower} shows two regimes under the case $\zeta \ge \tilde \Omega(\Delta / \sqrt{d})$. In the first regime $K \le \cO(|\cD|)$ where the decision set is large (e.g., $|\cD| = d^{100}$), any algorithm will suffer from a linear regret $\tilde \cO(\Delta K)$, which suggests that the regime cannot be efficiently learnable. In the second regime $K \ge \Omega(|\cD|)$, Theorem~\ref{thm:lower} suggests an $\tilde \Omega(|\cD|\Delta^{-1}\log(K))$ regret lower bound, which is matched by the multi-armed bandit algorithm with an upper bound $\tilde \cO(|\cD|\Delta^{-1}\log(K))$~\citep{lattimore2020bandit}. Therefore, in this easier regime, linear function approximation cannot provide any performance improvement and one can simply adopt the multi-armed bandit algorithm to learn the bandit model.
\end{remark}
\begin{remark}
Theorems~\ref{thm:main} and~\ref{thm:lower} provide a holistic picture about the role of misspecification in linear contextual bandits. Here we focus on the more difficult regime $K \le |\cD|$. In the regime $K \le |\cD|$, when $\zeta \le \tilde \cO(\Delta / \sqrt{d})$, Theorem~\ref{thm:main} suggests that the bandit problem is efficiently learnable, and our algorithm DS-OFUL can achieve a constant regret, which improves upon the logarithmic regret bound in the well-specified setting~\citep{abbasi2011improved}. On the other hand, when $\zeta \ge \tilde \Omega(\Delta / \sqrt{d})$, Theorem~\ref{thm:lower} provides a linear regret lower bound suggesting that the bandit model can not be efficiently learned.
\end{remark}
\section{Experiments}

\begin{table*}[htbp!]
    \centering
    \caption{Averaged cumulative regret and elapsed time of DS-OFUL over 8 runs. The \textbf{bold face} value indicates the best (low regret or low elapsed time) for all the algorithm configurations}
    \vspace{1em}
    \label{tab:1}
    \begin{tabular}{C{19em}|C{7em}|C{7em}|C{8em}}
    \toprule
     Algorithm Configuration,  ($\Gamma$) & Regret (mean$\pm$std.) & Regret in last 1k steps & Elapsed Time(sec)\\
     \midrule
     OFUL~\citep{abbasi2011improved}, $\Gamma=0$ & $405.4\pm76.5$ & $4.94$ & $15.06$\\ 
     DS-OFUL (Algorithm~\ref{alg:main}), $\Gamma = 0.02$ &  $326.5\pm68.0$ & $\mathbf{0.0}$ & $8.59$ \\
     DS-OFUL (Algorithm~\ref{alg:main}), $\Gamma = 0.05$ & $\mathbf{235.75 \pm 40.3}$ & $\mathbf{0.0}$ & $6.30$\\
     DS-OFUL (Algorithm~\ref{alg:main}), $\Gamma = 0.08$ & $411.6\pm566.7$ & $22.44$ & $5.97$\\
     DS-OFUL (Algorithm~\ref{alg:main}), $\Gamma = 0.13$ & $1789.5\pm1918.8$ & $173.67$ & $\mathbf{5.56}$\\
     Eq. (6) in \citet{lattimore2020learning} & $433.36 \pm 64$ & $1.79$ & $\ge 7\ \text{hrs.}$\\
     Robust Linear Bandit \citep{ghosh2017misspecified} & $831.5 \pm 880.4$ & $42.58$ & $12.85$\\
     SupLinUCB (Algorithm~\ref{alg:suplin}) & $747.9 \pm 329.5$ & $\mathbf{0.0}$ & $31.86$\\
     \bottomrule
    \end{tabular}\vspace{-0.7em}
\end{table*}
To verify the performance improvement by data selection using the UCB bonus in Algorithm~\ref{alg:main} and the effectiveness of the parameter-free algorithm Algorithm~\ref{alg:suplin}, we conduct experiments for bandit tasks on both synthetic and real-world datasets, which we will describe in detail below.

\subsection{Synthetic Dataset}
The synthetic dataset is composed as follows: we set $d = 16$ and generate parameter $\btheta^* \sim \cN(\zero, \Ib_d)$ and contextual vectors $\{\xb_i\}_{i=1}^N \sim \cN(\zero, \Ib_d)$ where $N = 100$. The generated parameter and vectors are later normalized to be $\|\btheta^*\|_2 = \|\xb_i\|_2 = 1$. The reward function is calculated by $r_i = \la \btheta^*, \xb_i\ra + \eta_i$ where $\eta_i \sim \text{Unif}\{-\zeta, \zeta\}$. 
The contextual vectors and reward function is fixed after generated. The random noise on the receiving rewards $\eps_t$ are sampled from the standard normal distribution. 

We set the misspecification level $\zeta = 0.02$ and verified that the sub-optimality gap over the $N$ contextual vectors $\Delta \approx 0.18$. We do a grid search for $\beta = \{1, 3, 10\}$, $\lambda = \{1, 3, 10\}$ and report the cumulative regret of Algorithm~\ref{alg:main} with different parameter $\Gamma = \{0, 0.02, {0.05}, 0.08, 0.18\}$ over 8 independent trials with total rounds $K = 10000$. It is obvious that when $\Gamma = 0$, our algorithm degrades to the standard OFUL algorithm \citep{abbasi2011improved} which uses data from all rounds into regression.

Besides the OFUL algorithm, we also compare with the algorithm (LSW) in Equation (6) of~\citet{lattimore2020learning} and the RLB in~\citet{ghosh2017misspecified} in Figure~\ref{fig:1} and Table~\ref{tab:1}. For \citet{lattimore2020learning}, the estimated reward is updated by $r(\xb) = \xb^\top\btheta_k + \beta \|\xb\|_{\Ub_k^{-1}} + \eps\sum_{s=1}^k|\xb^\top \Ub_k^{-1}\xb_s^{-1}|$. However, since the time complexity of the LSW algorithm is $\tilde \cO(K^2)$ due to the hardness of calculating $\eps\sum_{s=1}^k|\xb^\top \Ub_k^{-1}\xb_s^{-1}|$ incrementally w.r.t. $k$. In our setting it takes more than 7 hours for 10000 rounds. 

For the RLB algorithm in \citet{ghosh2017misspecified}, we did the hypothesis test for $k = 10$ rounds and then decided whether to use OFUL or multi-armed UCB. The results show that both LSW and RLB achieve a worse regret than OFUL since in our setting $\zeta$ is relatively small.

The result is shown in Figure~\ref{fig:1} and the average cumulative regret on the last round is reported in Table~\ref{tab:1} with its variance over 8 trials. We can see that by setting $\Gamma \approx \Delta  / \sqrt{d}  \approx 0.18 / \sqrt{16} \approx  0.05$, Algorithm~\ref{alg:main} can achieve less cumulative regret compared with OFUL ($\Gamma = 0$). The algorithm with a proper choice of $\Gamma$ also convergences to zero instantaneous regret faster than OFUL. It is also evident that a too large $\Gamma = 0.18 \approx \Delta$ will cause the algorithm to fail to learn the contextual vectors and induce a linear regret. Also, our algorithm shows that using a larger $\Gamma$ can significantly boost the speed of the algorithm by reducing the number of regressions needed in the algorithm.

Besides the performance improvement achieved by Algorithm~\ref{alg:main}, the experiments also demonstrates the effectiveness of Algorithm~\ref{alg:suplin}. As Table~\ref{tab:1} suggests, SupLinUCB achieves a zero cumulative regret over the last 1000 steps. However, as discussed in Remark~\ref{rm:3}, the total regret of SupLinUCB is much higher than the DS-OFUL and OFUL since it takes more samples to learn the first $l_{\Delta} - 1$ levels which is not used by DS-OFUL. This constant larger sample complexity could also be verified by a longer elapsed time for executing the SubLinUCB comparing to DS-OFUL.


\begin{figure}[htbp]
\centering
\subfigure[Cumulative regret comparison of DS-OFUL (with difference choices of $\Gamma$), SupLinUCB, \citet{lattimore2020learning} and Robust Linear Bandit~\citet{ghosh2017misspecified} over 10000 rounds. Results are averaged over 8 replicates.]{\label{fig:1}\includegraphics[width=0.49\textwidth]{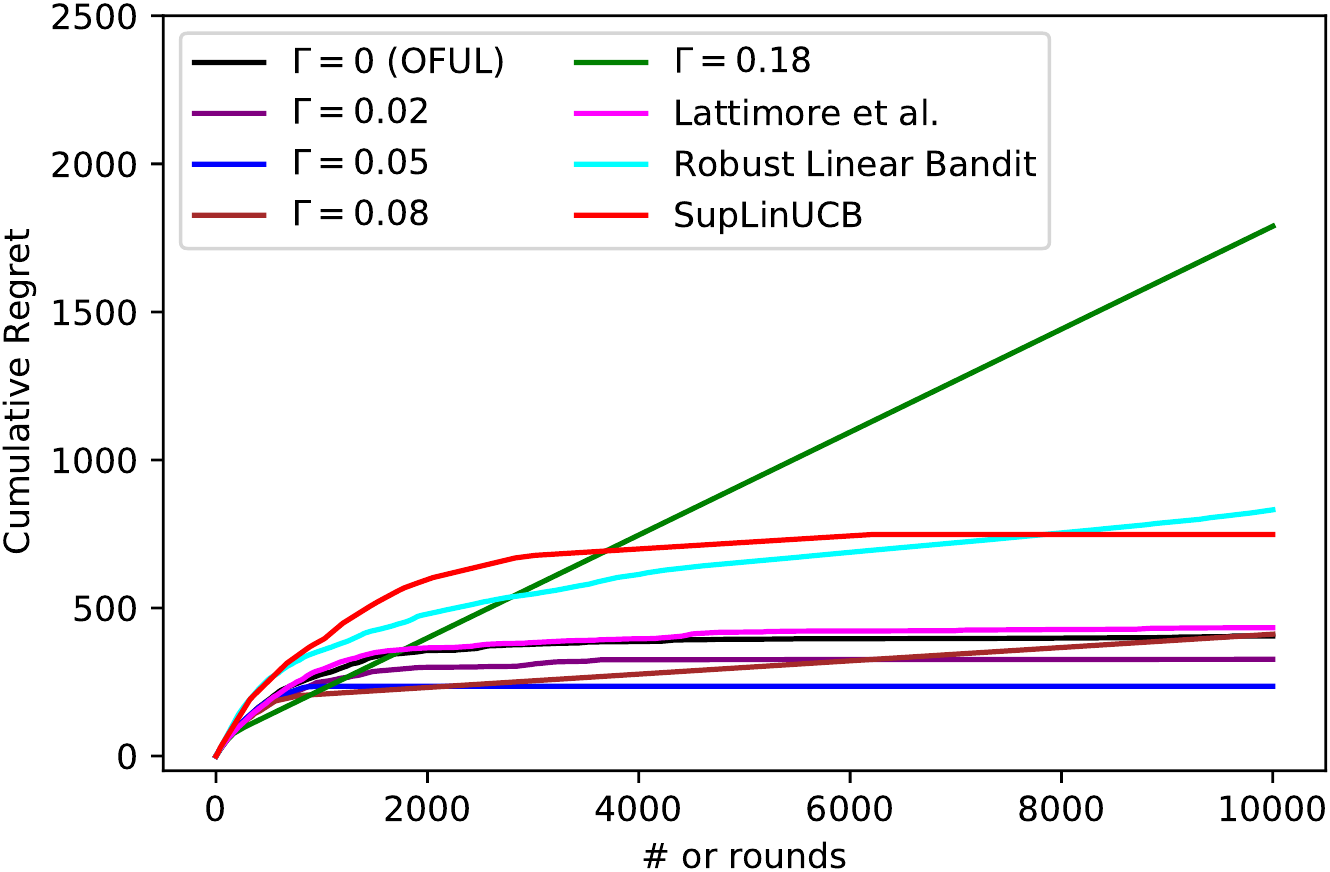}}
\subfigure[Cumulative regret of DS-OFUL on the Asirra dataset over 1M rounds with different $\Gamma$  under misspecification level $\zeta = 0.01$. Results are averaged over 8 runs. The cumulative regret of DS-OFUL (as well as OFUL) can be read from the y-axis on the left. The cumulative regret of SupLinUCB algorithm can be read from the y-axis on the right.]{\label{fig:2}\includegraphics[width=0.49\textwidth]{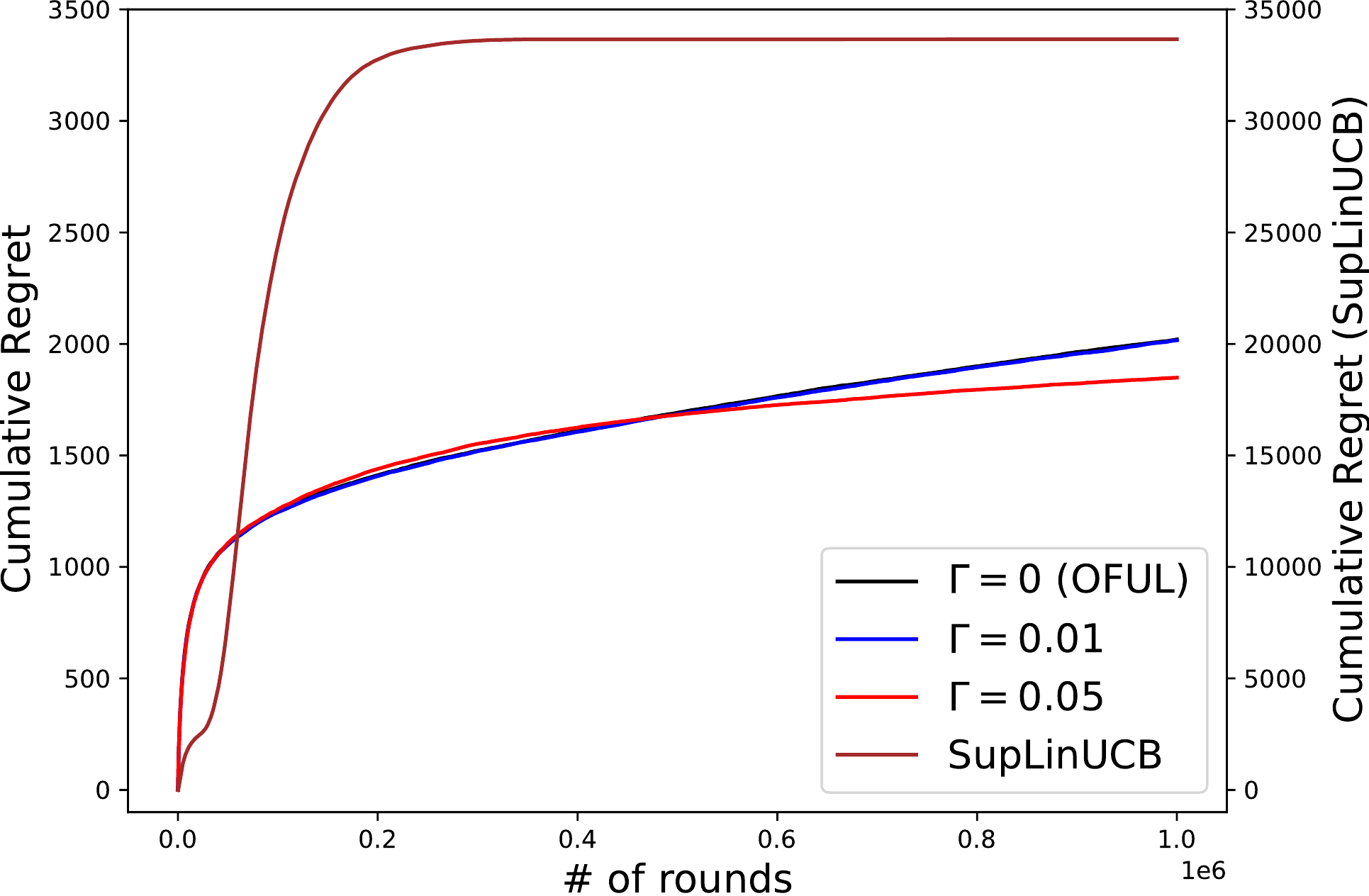}}
\caption{Cumulative regret on (a) synthetic dataset and (b) Asirra dataset }
\end{figure}

\subsection{Real-world Dataset}\label{sec:exp}
To demonstrate that the proposed algorithm can be easily applied to modern machine learning tasks, we carried out experiments on the Asirra dataset~\citep{asirra-a-captcha-that-exploits-interest-aligned-manual-image-categorization}. The task of agent is to distinguish the image of cats from the image of dogs. At each round $k$, the agent receives the feature vector $\bphi_{1, k} \in \RR^{512}$ of  a cat image and another feature vector $\bphi_{2, k} \in \RR^{512}$ of a dog image. Both feature vectors are generated using ResNet-18~\citep{he2016deep} pretrained on ImageNet~\citep{5206848}. We normalize $\|\bphi_{1, k}\|_2 = \|\bphi_{2, k}\|_2 = 1$. The agent is required to select the cat from these two vectors. It receives reward $r_t = 1$ if it selects the correct feature vector,  and receives $r_t = 0$ otherwise. It is trivial that the sub-optimality gap of this task is $\Delta = 1$. To better demonstrate the influence of misspecification on the performance of the algorithm, we only select the data with $|\bphi_i^\top\btheta^* - r_i| \le \zeta$ with $r_i = 1$ if it is a cat and $r_i = 0$ otherwise. $\btheta^*$ is a pretrained parameter on the whole dataset using linear regression $\btheta^* = \argmin_{\btheta} \sum_{i=1}^N (\bphi_i^\top\btheta - r_i)^2$, which the agent does not know.
For hyper-parameter tuning, we select $\beta = \{0.1, 0.3, 1\}$ and $\lambda = \{1, 3, 10\}$ by doing a grid search and repeat the experiments for 8 times over 1M rounds for each parameter configuration. As shown in Figure~\ref{fig:2}, when $\zeta = 0.01$, setting $\Gamma = 0.05 \approx \Delta / \sqrt{d}$ will eventually have a better performance comapred with OFUL algorithm (setting $\Gamma = 0$). On the other hand, the SupLinUCB algorithm (Algorithm~\ref{alg:suplin}) will suffer from a much higher, but constant regret bound, which is well aligned with our theoretical result especially Remark~\ref{rm:3}. We skip the Robust Linear Bandit~\citep{ghosh2017misspecified} algorithm since it is for stochastic linear bandit with fixed contextual features for each arm while here the contextual features are sampled and not fixed. The LSW (Equation~(6) in~\citet{lattimore2020learning}) is skipped due to the infeasible executing time.  

As a sensitivity analysis, we also set $\zeta = \{0.5, 0.1, 0.05\}$ to test the impact of misspecification on the performance of algorithm choices of $\Gamma$. More experiment configurations and results are deferred to Appendix~\ref{app:exp}.

\section{Conclusion and Future Work}
We study the misspecified linear contextual bandit from a gap-dependent perceptive. We propose an algorithm and show that if the misspecification level $\zeta \le \tilde \cO(\Delta / \sqrt d)$, the proposed algorithm, DS-ODUL, can achieve the same gap-dependent regret bound as in the well-specified case. Along with~\citet{lattimore2020learning, du2019good}, we provide a complete picture on the interplay between misspecification and sub-optimality gap, in which $\Delta / \sqrt{d}$ plays an important role on the phase transition of $\zeta$ to decide if the bandit model can be efficiently learned. 

Besides the aforementioned constant regret result, DS-OFUL algorithm requires the knowledge of sub-optimality ap $\Delta$. We prove that the SupLinUCB algorithm~\citep{chu2011contextual} can be viewed as a multi-level version of our algorithm and can also achieve a constant regret with our fine-grained analysis without the knowledge of $\Delta$. Experiments are conducted to demonstrate the performance of the DS-OFUL algorithm and verify the effectiveness of SupLinUCB algorithm.

The promising result suggests a few interesting directions for future research. For example, it would be interesting to incorporate the Lipschitz continuity or smoothness properties of the reward function to derive fine-grained results.


\appendix
\section{Experiment Details and Additional Results}\label{app:exp}
\subsection{Experiment Configuration}
\begin{wraptable}{r}{.5\textwidth}
    \centering
    \caption{The number of remaining data samples after data processing with expected misspecification level}
    \label{tab:2}
    \begin{tabular}{ccc}
    \toprule
     $\zeta$ & \# of cats & \# of dogs\\
     \midrule
     $\infty$ (without preprocessing) & $12500$ & $12500$ \\ 
     $0.5$ (linear separable) & $10316$ & $10511$ \\
     $0.1$ & $3182$ & $3248$ \\
     $0.05$ & $2408$ & $2442$ \\
     $0.01$ & $1886$ & $1905$ \\
     \bottomrule
    \end{tabular}
\end{wraptable}
The experiment on synthetic dataset is conducted on Google Colab with a 2-core Intel\textsuperscript{\textregistered} Xeon\textsuperscript{\textregistered} CPU @ 2.20GHz. The experiment on the real-world Asirra dataset~\citep{asirra-a-captcha-that-exploits-interest-aligned-manual-image-categorization} is conducted on an AWS p2-xlarge instance.

\subsection{Data Preprocessing for the Asirra Dataset}
To demonstrate how our algorithm can deal with different levels of misspecification, we do data preprocessing before feeding the data into the agent. As described in Section~\ref{sec:exp}, the remaining data with expected misspecification level $\zeta$ are shown in Table~\ref{tab:2}. It can be verified that even with the smallest misspecification level, there are still more than $10\%$ of the data is selected.

\subsection{Additional Result on the Asirra Dataset}
As a sensitivity analysis, we change the misspecification level in the preprocessing part in the Asirra dataset. The result is shown in Figure~\ref{fig:4}. This result suggests that when the misspecification is small enough, setting $\Gamma = \Delta /\sqrt{d}$ can deliver a reasonable result and SupLinUCB~\citet{chu2011contextual} can achieve a constant regret bound when $\zeta \le 0.1$. It is aligned with the parameter setting in our Theorem~\ref{thm:main} and the result in our Theorem~\ref{thm:main-sup}. Meanwhile, we found that when $\zeta = 0.5$, which means it is strictly larger than the threshold $\Delta / \sqrt{d}$, the algorithm cannot achieve 
a similar performance with of $\zeta < 0.1$, regardless of the setting of parameter $\Gamma$. This also verifies the theoretical understanding of how a large misspecification level will harm the performance of the algorithm.
\begin{figure}[htbp]
\centering
\subfigure[$\zeta = 0.5$]{\label{fig:a}\includegraphics[width=0.32\textwidth]{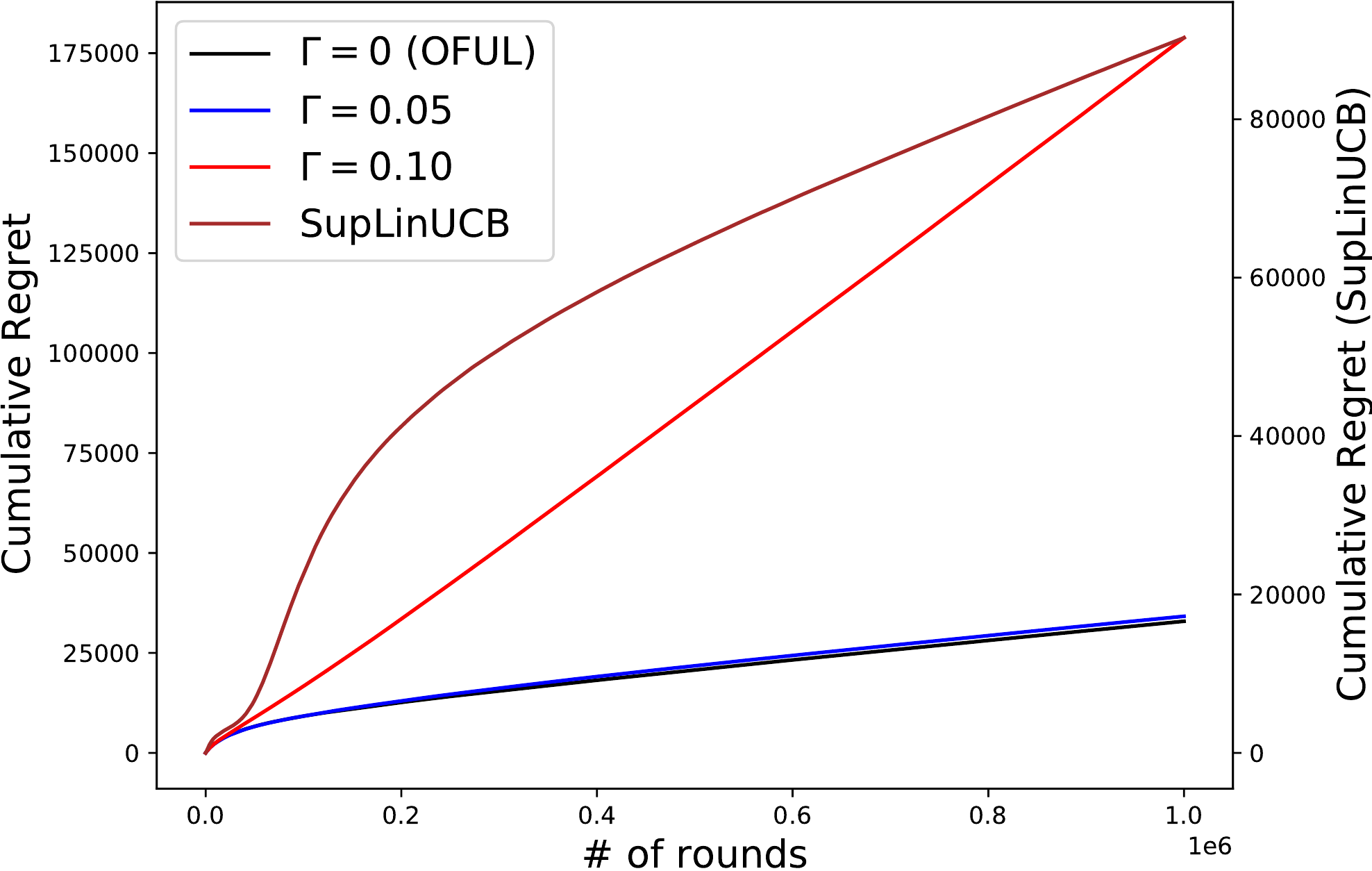}}
\subfigure[$\zeta = 0.1$]{\label{fig:b}\includegraphics[width=0.32\textwidth]{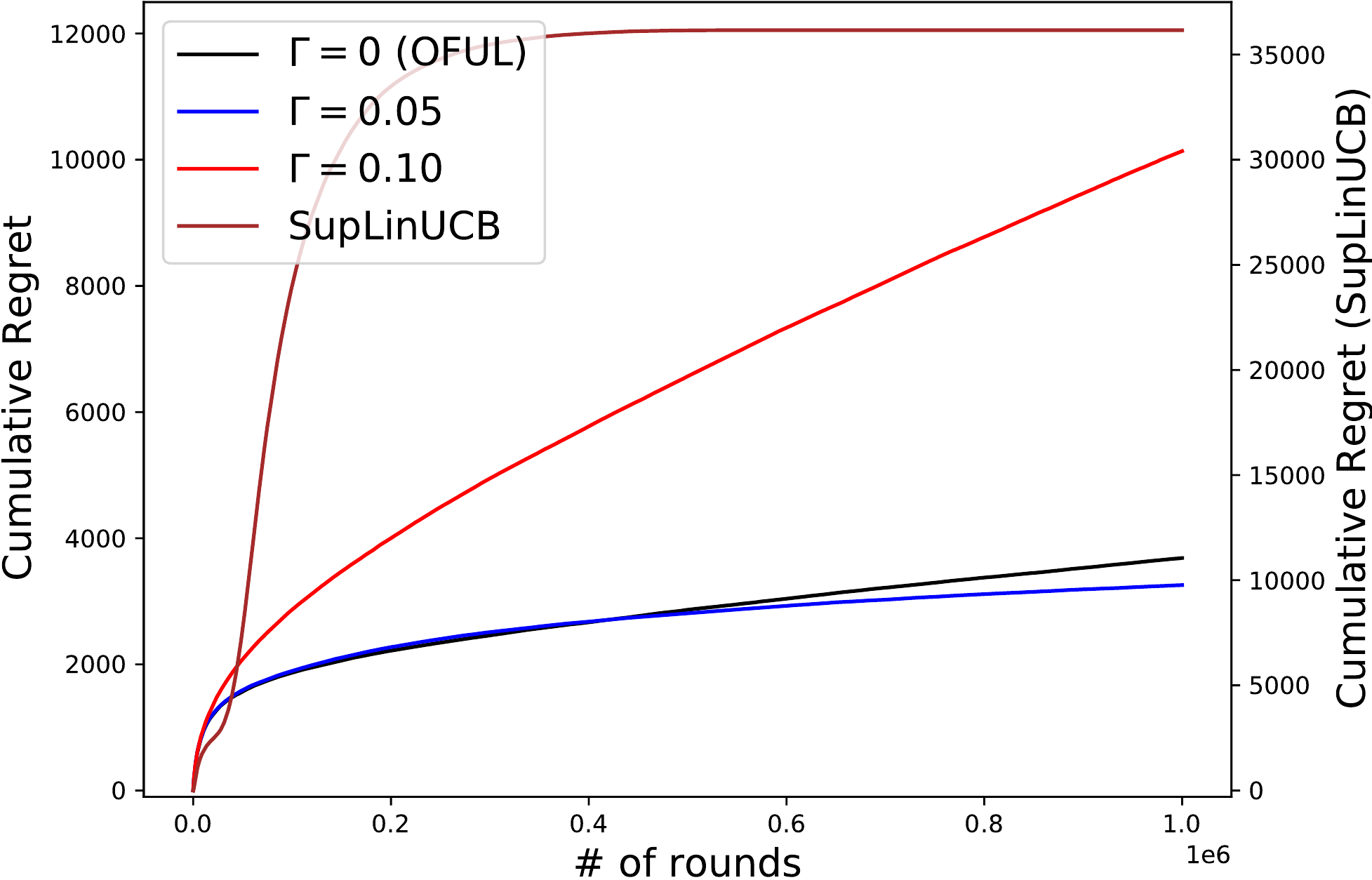}}
\subfigure[$\zeta = 0.05$]{\label{fig:c}\includegraphics[width=0.32\textwidth]{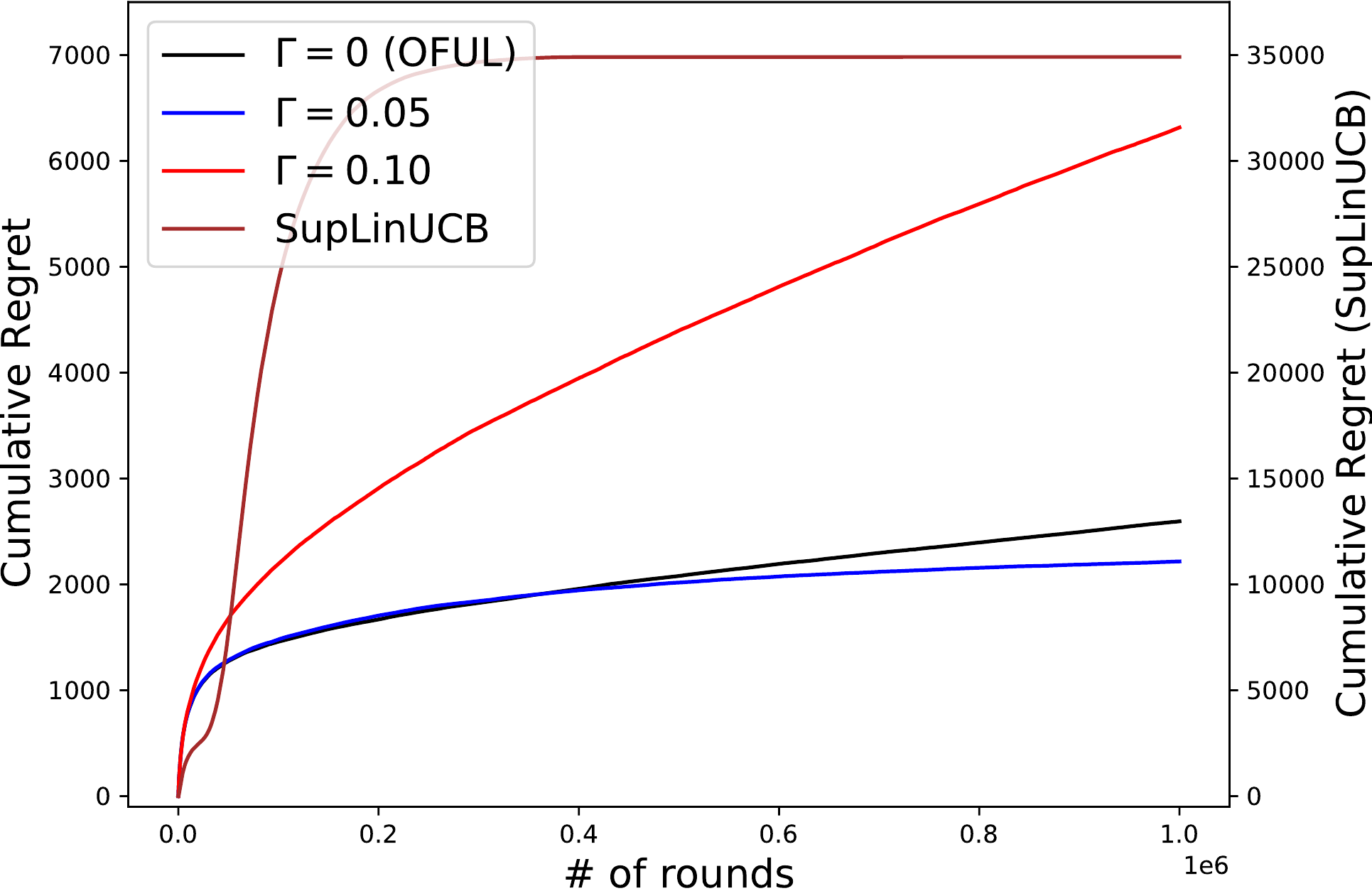}}
\caption{The performance of DS-OFUL under different misspecification levels $\zeta$. Results are averaged over 8 runs, with standard errors shown as shaded areas.}\label{fig:4}
\end{figure}

\section{Detailed Proof of Theorem~\ref{thm:main}}\label{app:proof0}
\label{sec:bandit}
In this section, we provide detailed proof for Theorem~\ref{thm:main}. First, we present a technical lemma to bound the total number of data used in the online linear regression in Algorithm \ref{alg:main}.
\begin{lemma}[Restatement of Lemma~\ref{lm:finite1}]\label{lm:finite}
Given $0 < \Gamma \le 1$, set $\lambda = B^{-2}$. For any $k \in [K]$, $|\cC_k| \le 16d\Gamma^{-2}\log(3LB\Gamma^{-1})$.
\end{lemma}
Lemma~\ref{lm:finite} suggests that up to $\tilde O(d\Gamma^{-2})$ contextual vectors have a UCB bonus greater than $\Gamma$. A similar result is also provided in~\citet{he2021uniformpac}, suggesting an $\tilde \cO(\Gamma^{-2})$ Uniform-PAC sample complexity. Lemma~\ref{lm:finite} also suggests that the numbers of data points added into the regression set $\cC$ is finite. Thus, the impact of the noise and the misspecification on the linear regression estimator can be well-controlled.

For a linear regression with up to $|\cC_k|$ data points, the  next lemma controls the prediction error under misspecification.

\begin{lemma}[Formal statement of Lemma~\ref{lm:decompose1}]\label{lm:decompose}
    Let $\lambda = B^{-2}$. For all $\delta > 0$, with probability at least $1 - \delta$, for all $\xb \in \RR^d, k \in [K]$, the prediction error is bounded by:
    \begin{align*}
        |\xb^\top(\btheta_k - \btheta^*)| & \le \left(1 + R\sqrt{2d\iota} + \zeta\sqrt{|\cC_k|}\right)\|\xb\|_{\Ub_k^{-1}},
    \end{align*}
    where $\iota = \log((d + |\cC_k|L^2B^2) / (d\delta))$ and $|\cC_k|$ is the total number of data used in regression at the $k$-th round.
\end{lemma}
Lemma~\ref{lm:decompose} provides a similar confidence bound as the well-specified linear contextual bandits algorithms like OFUL~\citep{abbasi2011improved}. Comparing the confidence radius here $\tilde \cO(R\sqrt{d} + \zeta\sqrt{|\cC_{k-1}|})$ with the conventional radius in OFUL $\tilde \cO(R\sqrt{d})$, one can find that there is an additional term $\zeta\sqrt{|\cC_{k}|}$ that is caused by the misspecification. If we directly use all data to do the regression, the resulting confidence radius will be in the order of $\tilde \cO(\sqrt{K})$ and therefore will lead to a $\cO(K\sqrt{\log K})$ regret bound (see Lemma 11 in~\citet{abbasi2011improved}). This makes the regret bound vacuous. In our algorithm, however, the confidence radius is only $\sqrt{|\cC_{k}|}$ where $|\cC_k|$ is bounded by Lemma~\ref{lm:finite}. As a result, our regret bound will not be vacuous (i.e., superlinear in $K$). 

When the misspecification level is well bounded by $\zeta = \tilde \cO(\Delta / \sqrt{d})$, the following corollary is a direct result of Lemmas~\ref{lm:decompose} by replacing the term $|\cC_k|$ with its upper bound provided in Lemma~\ref{lm:finite}.
\begin{corollary}\label{col:beta}
    Suppose $2\sqrt{d}\zeta\iota_1 \le \Delta$, let $\lambda = B^{-2}$ and $0 < \Gamma \le 1$. Let $\beta = 1 + 2\Delta\Gamma^{-1}\sqrt{\iota_2}/\iota_1 + R\sqrt{2d\iota_3}$ where $\iota_2 = \log(3LB\Gamma^{-1})$, $\iota_3 = \log((1 + 16L^2B^2\Gamma^{-2}\iota_2)/{\delta})$, then with probability at least $1 - \delta$, for all $\xb \in \RR^d, k \in [K]$, the estimation error for all $k\in[K]$ is bounded by: $|\xb^\top(\btheta_k - \btheta^*)| \le \beta\|\xb\|_{\Ub_k^{-1}}$.
\end{corollary}
\begin{proof}
By Lemma~\ref{lm:finite}, replacing $|\cC_k|$ with its upper bound yields
\begin{align*}
    |\xb^\top(\btheta_k - \btheta^*)| \le (1 + 4\sqrt d\zeta \Gamma^{-1}\sqrt{\iota_2} + R\sqrt{2d\iota_3})\|\xb\|_{\Ub_k^{-1}} \le \beta \|\xb\|_{\Ub_k^{-1}},
\end{align*}
where the second inequality is due to the condition $2\sqrt d\zeta \le \Delta / \iota_1$. 
\end{proof}
Next we introduce an auxiliary lemma controlling the instantaneous regret bound using the UCB bonus and the misspecification level.
\begin{lemma}[Formal statement of Lemma~\ref{lm:mis-informal}]\label{lm:mis}
Suppose Corollary~\ref{col:beta} holds, for all $k \in [K]$, the instantaneous regret at round $k$ is bounded by
\begin{align*}
    \Delta_k(\xb_k) = r^*_k - r(\xb_k) \le 2\zeta + 2\beta\|\xb_k\|_{\Ub_k^{-1}}.
\end{align*}
\end{lemma}

The next technical lemma from~\citet{he2021logarithmic} bounds the summation of a subset of the bonuses.
\begin{lemma}[Lemma 6.6,~\citealt{he2021logarithmic}]\label{lm:he}
For any subset $\cG = \{c_1, \cdots, c_i\} \subseteq \cC_K$, we have 
\begin{align*}
    \sum_{k \in \cG} \left\|\xb_k\right\|_{\Ub_k^{-1}}^2 \le 2d \log(1 + |\cG|L^2 / \lambda).
\end{align*}
\end{lemma}

The next auxiliary lemma is used to control the dominating terms.

\begin{lemma}\label{lm:aux1}
Let $\iota_1 = (24 + 18R)\log((72 + 54R)LB\sqrt{d}\Delta^{-1}) + \sqrt{8R^2\log(1 / \delta)}$, $\Gamma = \Delta / (2\sqrt{d}\iota_1)$, $\iota_2 = \log(3LB\Gamma^{-1}), \iota_3 = \log((1 + 16L^2B^2\Gamma^{-2}\iota_2) / \delta)$, we have $\iota_1 > 2 + 4\sqrt{\iota_2} + R\sqrt{2\iota_3}$.
\end{lemma}


Equipped with these lemmas, we can start the proof of Theorem~\ref{thm:main}.
\begin{proof}[Proof of Theorem~\ref{thm:main}] 
First, note that by setting $\Gamma = \Delta / (2\sqrt{d}\iota_1)$, the confidence radius $\beta$ becomes $1 + 4\sqrt{d\iota_2} + R\sqrt{2d\iota_3}$. Then our proof starts by assuming that Corollary~\ref{col:beta} holds with probability at least $1 - \delta$. We decompose the index set $[K]$ into two subsets. The first set is the set of not selected data $[K] \setminus \cC_K$, and the second set is the set of selected data $\cC_K$. We will bound the cumulative regret within these two sets separately.

First, for those non-selected data $k \notin \cC_k$, i.e. $\|\xb_k\|_{\Ub_k^{-1}} < \Gamma$, combining Lemma~\ref{lm:mis} with Corollary~\ref{col:beta} yields
\begin{align}
    r^*_k - r(\xb_k) < 2\zeta + 2\beta\Gamma = 2\zeta + \frac{\Delta}{\sqrt{d}\iota_1} + \frac{\sqrt{2\iota_3}R\Delta}{\iota_1} + \frac{4\Delta\sqrt{\iota_2}}{\iota_1},\label{eq:step1}
\end{align}
where $\iota_1, \iota_2, \iota_3$ are the same as Theorem~\ref{thm:main}, and the equality is due to $\Gamma = \Delta / (2\sqrt{d}\iota_1)$. When misspecification condition $2\sqrt{d}\zeta \le \Delta / \iota_1$ holds, \eqref{eq:step1} suggests that 
\begin{align}
    r_k^* - r(\xb_k) < \frac{2\Delta}{\sqrt{d}\iota_1} + \frac{4\Delta\sqrt{\iota_2}}{\iota_1} + \frac{\sqrt{2\iota_3}R\Delta}{\iota_1}.\label{eq:rev1}
\end{align}
Lemma~\ref{lm:aux1} suggests that when $\iota_1 = (24 + 18R)\log((72 + 54R)LB\sqrt{d}\Delta^{-1}) + \sqrt{8R^2\log(1 / \delta)}$ $\iota_1 > 2 + 4\sqrt{\iota_2} + R\sqrt{2\iota_3}$, \eqref{eq:rev1} yields that the instantaneous regret $r_k^* - r(\xb_k) < \Delta$ at round $k$. By Definition~\ref{def:bandit-gap}, the instantaneous regret is zero for all $k \notin \cC_k$, indicating the non-selected data incur zero instantaneous regret. 


In addition, Lemma~\ref{lm:mis} suggests that the instantaneous regret for those $k \in \cC_K$ is bounded by
\begin{align}
    \sum_{k \in \cC_K} r^*_k - r(\xb_k) &\le \sum_{k \in \cC_K} \left(2\beta \|\bphi_k\|_{\Ub_k^{-1}} + 2\zeta\right)\notag \\
    &\le 2\beta\sqrt{|\cC_K|}\sqrt{\sum_{k \in \cC_K}\|\bphi_k\|_{\Ub_k^{-1}}^2}+ 2|\cC_K|\zeta \notag \\
    &\le 8\beta\Gamma^{-1}\sqrt{d\iota_2}  \sqrt{2d\log(1 + 16d\Gamma^{-2}\iota_2)} +  32\zeta d\Gamma^{-2}\iota_2 \notag \\
    &\le 16\beta\sqrt{2d^3\iota_2\log(1 + 16d\Gamma^{-2}\iota_2)}\iota_1 / \Delta + 64\sqrt{d^3}\iota_1\iota_2 / \Delta \notag \\
    &\le 32\beta\sqrt{2d^3\iota_2\log(1 + 16d\Gamma^{-2}\iota_2)}\iota_1 / \Delta,\label{eq:step-finalb}
\end{align}
where the second inequality follows the Cauchy-Schwarz inequality, the third one yields from Lemma~\ref{lm:he} while the fourth utilizes the fact that $\Gamma = \Delta / (2\sqrt{d}\iota_1)$ and $\zeta \le \Delta / (2\sqrt{d}\iota_1)$. The last one is due to the fact that the second term in the fourth inequality is dominated by the first one. 


To warp up, the cumulative regret can be decomposed by
\begin{align*}
    \text{Regret}(K) &= \sum_{k \notin \cC_K}(r_k^* - r(\xb_k)) + \sum_{k \in \cC_K }(r_k^* - r(\xb_k)) \le 0 + \frac{32\beta\sqrt{2d^3\iota_2\log(1 + 16d\Gamma^{-2}\iota_2)}\iota_1}\Delta,
\end{align*}
where the first two zeros are given by the fact that for $k \notin \cC_K$, we have $r_k^* - r(\xb_k) = 0$. the regret bound for $k \in \cG$ is given by \eqref{eq:step-finalb}.
\end{proof}
\section{Proof of Technical Lemmas in Appendix~\ref{sec:bandit}}\label{app:proof}
\subsection{Proof of Lemma~\ref{lm:finite}}
To prove this lemma, we introduce the well-known elliptical potential lemma~\citep{abbasi2011improved}
\begin{lemma}[Lemma 11,~\citealt{abbasi2011improved}]\label{lm:11}
    Let $\{\bphi_i\}_{i=1}^I$ be a sequence in $\RR^d$, define $\Ub_i = \lambda \Ib + \sum_{j=1}^i\bphi_j\bphi_j^\top$, then 
    \begin{align*}
        \sum_{i=1}^I \min\left\{1, \|\bphi_i\|_{\Ub_{i - 1}^{-1}}^2\right\} \le 2d \log\left(\frac{\lambda d + IL^2}{\lambda d}\right).
    \end{align*}
\end{lemma}
The following auxiliary lemma and its corollary are useful
\begin{lemma}[Lemma A.2,~\citealt{shalev2014understanding}]\label{lm:xlogx}
Let $a \ge 1$ and $b > 0$. Then $x \ge 4a\log(2a) + 2b$ yields $x \ge a\log(x) + b$.
\end{lemma}

Lemma~\ref{lm:xlogx} can easily indicate the following lemma.

\begin{lemma}\label{col:xlogx} Let $a \ge 1$. Then $x \ge 4\log(2a) + a^{-1}$ yields $x \ge \log(1 + ax)$.
\end{lemma}
\begin{proof}
Let $y = 1 + ax, x = (y - 1) / a$. Then $x \ge 4\log(2a) + a^{-1}$ is equivalent with $y \ge 4a\log(2a) + 2$. By Lemma~\ref{lm:xlogx}, this implies $y \ge a\log(y) + 1$ which is exactly $x \ge \log(1 + ax)$.
\end{proof}

Equipped with these technical lemmas, we can start our proof.
\begin{proof}[Proof of Lemma~\ref{lm:finite}]
Since the cardinality of set $\cC_k$ is monotonically increasing w.r.t. $k$, we fix $k$ to be $K$ in the proof and only provide the bound of $\cC_K$. For all selected data $k \in \cC_K$,  we have $\|\bphi_k\|_{\Ub_k^{-1}} \ge \Gamma$. Therefore, when $\Gamma \le 1$, the summation of the bonuses over data $k \in \cC_K$ is lower bounded by
\begin{align}
    \sum_{k \in \cC_K} \min\left\{1, \|\bphi_k\|_{\Ub_{k}^{-1}}^2\right\} \ge |\cC_K|\min\{1, \Gamma^2\} = |\cC_K|\Gamma^2.\label{eq:lowerbounds}
\end{align}
On the other hand, Lemma~\ref{lm:11} implies
\begin{align}
    \sum_{k \in \cC_K} \min\left\{1, \|\bphi_k\|_{\Ub_k^{-1}}^2\right\} \le 2d \log\left(\frac{\lambda d + |\cC_K|L^2}{\lambda d}\right)\label{eq:upperbounds}.
\end{align}
Combining~\eqref{eq:upperbounds} and~\eqref{eq:lowerbounds}, the total number of the selected data points $|\cC_K|$ is bounded by
\begin{align*}
    \Gamma^2|\cC_K| \le 2d \log\left(\frac{\lambda d + |\cC_K|L^2}{\lambda d}\right).
\end{align*}
This result can be re-organized as 
\begin{align}
    \frac{\Gamma^2|\cC_K|}{2d} \le \log\left(1 + \frac{2L^2}{\Gamma^2\lambda}\frac{\Gamma^2|\cC_K|}{2d}\right).\label{eq:thres}
\end{align}
Let $\lambda = B^{-2}$ and since $2L^2B^2 \ge 2 \ge \Gamma^2$, by Lemma~\ref{col:xlogx}, if
\begin{align*}
    \frac{\Gamma^2 |\cC_K|}{2d} > 4\log\left(\frac{4L^2B^2}{\Gamma^2}\right) + 1 \ge 4\log\left(\frac{4L^2B^2}{\Gamma^2}\right) + \frac{\Gamma^2}{2L^2B^2},
\end{align*}
then \eqref{eq:thres} will not hold. Thus the necessary condition for~\eqref{eq:thres} to hold is
\begin{align*}
    \frac{\Gamma^2 |\cC_K|}{2d} \le 4\log\left(\frac{4L^2B^2}{\Gamma^2}\right) + 1 = 8\log\left(\frac{2LB}{\Gamma}\right) + \log(e) =  8\log\left(\frac{2LBe^{\frac18}}{\Gamma}\right) < 8\log\left(\frac{3LB}{\Gamma}\right).
\end{align*}
By basic calculus we get the claimed bound for $|\cC_K|$ and complete the proof.
\end{proof}
\subsection{Proof of Lemma~\ref{lm:decompose}}
The proof follows the standard technique for linear bandits, we first introduce the self-normalized bound for vector-valued martingales from~\citet{abbasi2011improved}.
\begin{lemma}[Theorem 1,~\citealt{abbasi2011improved}]\label{lm:thm1}
Let $\{\cF_t\}_{t=0}^\infty$ be a filtration. Let $\{\eps_t\}_{t=1}^\infty$ be a real-valued stochastic process such that $\eps_t$ is $\cF_t$-measurable and $\eps_t$ is conditionally $R$-sub-Gaussian for some $R \ge 0$. Let $\{\bphi_t\}_{t=1}^\infty$ be an $\RR^d$-valued stochastic process such that $\bphi_t$ is $\cF_{t-1}$ measurable and $\|\bphi\|_2 \le L$ for all $t$. For any $t \ge 0$, define $\Ub_t = \lambda \Ib + \sum_{k=1}^t\bphi_k\bphi_k$. Then for any $\delta > 0$, with probability at least $1 - \delta$, for all $t \ge 0$
\begin{align*}
    \left\|\sum_{k=1}^t\bphi_k\eps_k\right\|_{\Ub_t^{-1}}^2 \le 2R^2\log\left(\frac{\sqrt{\det(\Ub_t)}}{\sqrt{\det(\Ub_0)}\delta}\right).
\end{align*}
\end{lemma}

\begin{lemma}[Lemma 8, \citealt{Zanette2020LearningNO}]
\label{lm:lemma8}
Let $\{\mathbf{a}_i\}_{i=1}^d$ be any sequence of vectors in $\mathbb{R}^d$ and $\{b_i\}_{i=1}^d$ be any sequence of scalars such that $|b_i| \leq \zeta$. For any $\lambda > 0$:
\begin{align*}
    \left\| \sum_{i=1}^{n} \ab_i b_i \right\|^2_{\left[\sum_{i=1}^{n} \ab_i \ab_i^{\top} + \lambda \Ib \right]^{-1} } \leq n\zeta^2.
\end{align*}
\end{lemma}
The next lemma is to bound the perturbation of the misspecification
\begin{lemma}\label{lm:c4}
Let $\{\eta_k\}_k$ be any sequence of scalars such that $|\eta_k| \le \zeta$ for any $k \in [K]$. For any index subset $\cC \subseteq [K]$, define $\Ub = \lambda \Ib + \sum_{k \in \cC} \xb_k\xb_k^\top$, then for any $\xb \in \RR^d$, we have
\begin{align*}
    \bigg|\xb^\top \Ub^{-1} \sum_{k \in \cC} \xb_k\eta_k\bigg| \le \zeta \sqrt{|\cC|}\|\xb\|_{\Ub^{-1}}.
\end{align*}
\end{lemma}
\begin{proof}
By Cauchy-Schwartz inequality we have
    \begin{align*}
        \left| \xb^{\top}\Ub^{-1} \sum_{k \in \mathcal{C}} \xb_k \eta_k \right| \le \| \xb\|_{\Ub^{-1}} \left\| \sum_{k \in \mathcal{C}} \xb_k \eta_k \right\|_{\Ub^{-1}} \le \zeta \sqrt{|\mathcal{C}|} \|\xb\|_{\Ub^{-1}},
    \end{align*}
    where the second inequality dues to lemma \ref{lm:lemma8}.
\end{proof}
The next lemma is the Determinant-Trace inequality.
\begin{lemma}\label{lm:dt}
Suppose sequence $\{\xb_k\}_{k=1}^K \subset \RR^d$ and for any $k \in [K]$, $\|\xb_k\|_2 \le L$. For any index subset $\cC \subseteq [K]$, define $\Ub = \lambda \Ib + \sum_{k \in \cC} \xb_k\xb_k^\top$ for some $\lambda > 0$, then $\det(\Ub) \le (\lambda + |\cC|L^2/d)^d$.
\end{lemma}
\begin{proof}
The proof of this lemma is almost the same as Lemma 10 in~\citet{abbasi2011improved} by replacing the index set $[K]$ with any subset $\cC$. We refer the readers to~\citet{abbasi2011improved} for details.
\end{proof}
Equipped with these lemmas, we can start our proof.
\begin{proof}[Proof of Lemma~\ref{lm:decompose}]
For any $k \in [K]$, considering the data samples $k' \in \cC_{k-1}$ used for regression at round $k$. Following the update rule of $\Ub_k$ and $\btheta_k$ yields
\begin{align*}
    \Ub_k(\btheta_k - \btheta^*) &= \Ub_k\Ub_k^{-1}\bigg(\sum_{{k'} \in \cC_{k-1}}\xb_{k'}r_{k'}\bigg) - \bigg(\lambda \Ib +  \sum_{{k'} \in \cC_{k-1}} \xb_{k'} \xb_{k'}^\top\bigg)\btheta^*\\
    &= \sum_{{k'} \in \cC_{k-1}}\xb_{k'}r_{k'} - \lambda \btheta^* - \sum_{{k'} \in \cC_{k-1}} \xb_{k'} \xb_{k'}^\top \btheta^* \\
    &=-\lambda \btheta^* + \sum_{{k'} \in \cC_{k-1}}\xb_{k'}(r_{k'} - \xb_{k'}^\top\btheta^*) \\
    &= -\lambda\btheta^* + \sum_{{k'} \in \cC_{k-1}}\xb_{k'}\eps_{k'} + \sum_{{k'} \in \cC_{k-1}}\xb_{k'}\eta_{k'},
\end{align*}
where the first equation is due to the fact that $\Ub_k = \lambda \Ib + \sum_{k' \in \cC_{k-1}} \xb_k\xb_k^\top$ and $\btheta_k = \Ub_k^{-1}\sum_{k' \in \cC_{k-1}} \xb_{k'}r_{k'}$. The last equation follows the fact that $r_{k'}$ is generated from $r_{k'} = r(\xb_{k'}) + \eps_{k'} = \xb_{k'}^\top\btheta^* + \eta(\xb_{k'}) + \eps_{k'}$, where we denote $\eta(\xb_{k'})$ as $\eta_{k'}$ for the model misspecification error and $\eps_{k'}$ is the random noise. Therefore, consider any contextual vector $\xb \in \RR^d$, we have
\begin{align*}
    \left|\xb^\top(\btheta_k - \btheta^*)\right| &= \left|\xb^\top\Ub_k^{-1}\Ub_k(\btheta_k - \btheta^*)\right| \\
    &\le \lambda\underbrace{\left|\xb^\top\Ub_k^{-1}\btheta^*\right|}_{q_1} + \underbrace{\bigg|\xb^\top\Ub_k^{-1}\sum_{k' \in \cC_{k-1}}\bphi_{k'}\eps_{k'}\bigg|}_{q_2} + \underbrace{\bigg|\xb^\top\Ub_k^{-1}\sum_{k' \in \cC_{k-1}}\bphi_{k'}\eta_{k'}\bigg|}_{q_3},
\end{align*}
where the inequality is due to the triangle inequality. Lemma~\ref{lm:c4} yields $q_3 \le \zeta \sqrt{|\cC_{k-1}|}\|\xb\|_{\Ub_k^{-1}}$. From the fact that $|\xb^\top \Ab\yb| \le \|\xb\|_{\Ab}\|\yb\|_{\Ab}$, we can bound term $q_1$ by
\begin{align}
    q_1 \le \|\xb\|_{\Ub_k^{-1}} \|\btheta^*\|_{\Ub_k^{-1}} \le \lambda^{-1/2}B\|\xb\|_{\Ub_k^{-1}}. \label{eq:q1}
\end{align}
where the last inequality is due to the fact that $\Ub_k^{-1} \preceq \lambda^{-1}\Ib$. Term $q_2$ is also bounded as
\begin{align}
    q_2 \le \|\xb\|_{\Ub_k^{-1}} \Bigg\|\sum_{k' \in \cC_{k-1}} \xb_{k'} \eps_{k'}\Bigg\|_{\Ub_{k}^{-1}} = \|\xb\|_{\Ub_k^{-1}}\underbrace{\Bigg\|\sum_{k' = 1}^K \ind\left[k' \in \cC_{k-1}\right]\xb_{k'} \eps_{k'}\Bigg\|_{\Ub_{k}^{-1}}}_{I_1}, \label{eq:q2}
\end{align}
where the second equation uses the indicator function to rewrite the summation over subset $\cC_{k-1}$. Denoting $\yb_{k'} = \ind\left[k' \in \cC_{k-1}\right]\xb_{k'}$, noticing that $\|\yb_{k'}\|_2 \le \|\xb_{k'}\|_2 \le L$ and 
\begin{align*}
    \Ub_k = \sum_{k' \in \cC_{k-1}} \xb_{k'}\xb_{k'}^\top = \sum_{k' = 1}^K\ind\left[k' \in \cC_{k-1}\right] \xb_{k'}\xb_{k'}^\top = \sum_{k' = 1}^K \yb_{k'}\yb_{k'}^\top,
\end{align*}
by Lemma~\ref{lm:thm1}, $I_1$ can be further bounded by
\begin{align}
    I_1 \le \sqrt{2R^2\log\left(\frac{\sqrt{\det(\Ub_k)}}{\sqrt{\det(\Ub_0)}\delta}\right)} \le R\sqrt{2\log\left(\frac{\det(\Ub_k)}{\det(\Ub_0)\delta}\right)} = R\sqrt{2\log\left(\frac{\det(\Ub_k)}{\lambda^d\delta}\right)},\label{eq:rev3}
\end{align}
where the second inequality follows the fact that $\det(\Ub_k) \ge \det (\Ub_0) = \lambda^d$. Notice that $\Ub_k = \lambda \Ib + \sum_{k' \in \cC_{k-1}} \xb_{k'}\xb_{k'}^\top$. Lemma~\ref{lm:dt} suggests that $\det(\Ub_k) \le (\lambda + |\cC_{k-1}|L^2 / d)^d$, plugging this into~\eqref{eq:rev3}, we obtain
\begin{align*}
    I_1 \le R\sqrt{2\log\left(\frac{(\lambda + |\cC_{k-1}|L^2 / d)^d}{\lambda^d\delta}\right)} \le R\sqrt{2d\log\left(\frac{d\lambda + |\cC_{k-1}|L^2}{d\lambda \delta}\right)}.
\end{align*}

Plugging the bound of $I_1$ into ~\eqref{eq:q2} and combining with~\eqref{eq:q1} and Lemma~\ref{lm:c4} together, replacing $|\cC_{k-1}|$ with its upper bound $|\cC_K|$ we have with probability at least $1 - \delta$, for all $k \in [K], \xb \in \RR^d$, 
\begin{align*}
    |\xb^\top(\btheta_k - \btheta^*)| \le \left(R\sqrt{2d\log\left(\frac{d\lambda + |\cC_K|L^2}{d\lambda\delta}\right)} + B\lambda^{-1/2} + \zeta\sqrt{|\cC_K|}\right)\|\bphi\|_{\Ub_k^{-1}}.
\end{align*}
Letting $\lambda = B^{-2}$ we get the claimed results.
\end{proof}
\subsection{Proof of Lemma~\ref{lm:mis}}
\begin{proof}
According to the definition of expected reward function $r(\xb)$, we have for all $k \in [K]$, suppose the condition in Lemma~\ref{lm:decompose} holds, then
\begin{align*}
    r^*_k - r_k &= \eta(\xb_k^*) - \eta(\xb_k) + \left(\xb^*_k\right)^\top\btheta^* - \xb_k^\top\btheta^*\\
    &\le 2\zeta + \left(\xb^*_k\right)^\top\btheta^* - \xb_k^\top\btheta^*\\
    &= 2\zeta + \left(\xb^*_k\right)^\top\btheta_k + \left(\xb^*_k\right)^\top(\btheta^* - \btheta_k) - \xb_k^\top\btheta_k + \xb_k^\top( \btheta_k-\btheta^* )\\
    &\le 2\zeta + \left(\xb^*_k\right)^\top\btheta_k + \beta\|\xb^*_k\|_{\Ub_k^{-1}} - \xb_k^\top\btheta_k + \beta\|\xb_k\|_{\Ub_k^{-1}}\\
    &\le 2\zeta + \xb_k^\top\btheta_k + \beta\|\xb_k\|_{\Ub_k^{-1}} - \xb_k^\top\btheta_k + \beta\|\xb_k\|_{\Ub_k^{-1}}\\
    &\le 2\zeta + 2\beta\|\xb_k\|_{\Ub_k^{-1}},
\end{align*}
where the first inequality utilize the fact that $|\eta(\xb)| \le \zeta$ for all $\xb \in \cD_k$, the second inequality follows from Corollary~\ref{col:beta}, the third inequality is due to the fact that $\xb_k = \argmax_{\xb \in \cD_k} \xb^\top\btheta_k + \beta \|\xb\|_{\Ub_k^{-1}}$, which is executed in Line~\ref{ln:decision} of Algorithm~\ref{alg:main}.
\end{proof}

\subsection{Proof of Lemma~\ref{lm:aux1}}
\begin{proof}
First it is clear to see that $\sqrt{2\iota_3} = \sqrt{2\log(1 + 16L^2B^2\Gamma^{-2}\iota_2) + 2\log(1 / \delta)}$. Using the fact that $\sqrt{a + b} \le \sqrt{a} + \sqrt{b}$, it can be further bounded by
\begin{align*}
    \sqrt{2\iota_3} \le \sqrt{2\log(1 + 16L^2B^2\Gamma^{-2}\iota_2)} + \sqrt{2\log(1 / \delta)}.
\end{align*}
Assuming $L \ge 1, B \ge 1, \Gamma = \Delta / (2\sqrt{d}\iota_1) \le 1$ yields $LB\Gamma^{-1} \ge 1$, then by basic calculus one can verify that 
\begin{align*}
    2 + 4\sqrt{\iota_2} \le 6\log(3LB\Gamma^{-1}), \quad \sqrt{2\log(1 + 16L^2B^2\Gamma^{-2}\iota_2)} \le 3\log(3LB\Gamma^{-1}),
\end{align*}
therefore we have that 
\begin{align*}
    2 + 4\sqrt{\iota_2} + R\sqrt{2\iota_3} &\le (6 + 3R)\log(3LB\Gamma^{-1}) + \sqrt{2\log(1 / \delta)}R \\
    &= (6 + 3R)\log(6LB\sqrt{d}\Delta^{-1}\iota_1) + \sqrt{2\log(1 / \delta)}R,
\end{align*}
where the last equality is from the fact that $\Gamma = \Delta / (2\sqrt{d}\iota_1)$. Lemma~\ref{lm:xlogx} suggests that the necessary condition for
\begin{align}
    \underbrace{(6LB\sqrt{d}\Delta^{-1})\iota_1}_{x} \ge \underbrace{(6LB\sqrt{d}\Delta^{-1})(6 + 3R)}_{a}\log(6LB\sqrt{d}\Delta^{-1}\iota_1) + \underbrace{(6LB\sqrt{d}\Delta^{-1})\sqrt{2\log(1 / \delta)}R}_{b}
\end{align}
is that 
\begin{align*}
    (6LB\sqrt{d}\Delta^{-1})\iota_1 &\ge 4(6LB\sqrt{d}\Delta^{-1})(6 + 3R)\log(2(6LB\sqrt{d}\Delta^{-1})(6 + 3R)) \\
    &\quad + 2(6LB\sqrt{d}\Delta^{-1})\sqrt{2\log(1 / \delta)}R,
\end{align*}
which suggests that setting
\begin{align*}
    \iota_1 = (24 + 18R)\log((72 + 54R)LB\sqrt{d}\Delta^{-1}) + \sqrt{8R^2\log(1 / \delta)}
\end{align*}
implies the fact that $\iota_1 \ge 2 + 4\sqrt{\iota_2} + R\sqrt{2\iota_3}$
\end{proof}

\section{Detailed Proof of Theorem~\ref{thm:main-sup}}\label{sec:bandit-sup}

The first lemma shows that the contexts selected to $l$-th level are bounded independent from $K$
\begin{lemma}[Restatement of Lemma~\ref{lm:set1}]\label{lm:set}
Set $\lambda = B^{-2}$. For any $k \in [K]$ and $l > 0$, $|\cC_k^l| \le 16d4^l\iota_1(l)$ where $\iota_1(l) = \log\left(3LB2^l\right)$.
\end{lemma}
\begin{proof}
The proof is similar to the proof of Lemma~\ref{lm:finite} by repalcing $\Gamma = 2^{-l}$.
\end{proof}

The next lemma provides a fluctuation control as well as the concentration in the ridge regression
\begin{lemma}[Restatement of Lemma~\ref{lm:concent1}]\label{lm:concent}
Set $\lambda = B^{-2}$. For any level $l > 0$, for any $\delta > 0$, with probability at least $1 - \delta$, for all $k \in [K]$, the estimation error is bounded by
\begin{align*}
    \left|\xb^\top (\btheta_{k}^l - \btheta^*)\right| \le \left(1 + R\sqrt{2d\iota_{2}(l)} + \zeta \sqrt{\left|\cC_{k}^l\right|}\right)\|\xb\|_{(\Ub_k^l)^{-1}},
\end{align*}
for all $\xb$ such that $\|\xb\|_2 \le L$, where $\iota_{2}(l) = \log((d + |\cC_{k}^l| L^2 B^2) / (d \delta))$.
\end{lemma}
\begin{proof}
The proof is similar to the proof of Lemma~\ref{lm:decompose}
\end{proof}

Combining Lemma~\ref{lm:set} and Lemma~\ref{lm:concent}, we have the following corollary.

\begin{corollary}\label{col:beta-sup}
Set $\lambda = B^{-2}$. For any $\delta > 0$, with probability at least $1 - \delta$, for all round $k \in [K]$ and any level $l > 0$, the prediction error is bounded by 
\begin{align*}
    \left|\xb^\top (\btheta_{k}^l - \btheta^*)\right| \le \left(\beta(l) + 4\zeta 2^l \sqrt{d\iota_{1}(l)}\right)\|\xb\|_{(\Ub_k^l)^{-1}},
\end{align*}
for all $\xb$ such that $\|\xb\|_2 \le L$, where $\beta(l) = 1 + R\sqrt{2d\iota_{2}(l)}$, $\iota_{2}(l) = \log((d2^{l} + 16 L^2 B^2 8^l\iota_{1}(l)) / (d \delta))$, and $\iota_{1}(l) = \log\left(3LB2^l\right)$.
\end{corollary}
\begin{proof}
The proof is simply by plugging the result in Lemma~\ref{lm:set} into Lemma~\ref{lm:concent} and replacing the $\delta$ with $\delta / 2^l$. By the union bound over $l \in \NN^+$ and the fact that $\sum_{l=1}^\infty \delta/2^{l} = \delta$ yields the claimed result. 
\end{proof}

Now, we are about to control $\cD_k^l$, which means here we only consider the case where $\|\xb\|_{(\Ub_k^l)^{-1}} \le 2^{-l}$ for all $\xb \in \cD_k^l$ and assuming the high-probability event in previous subsection always holds. 
The following lemma suggests that the decision set always keeps a nearly optimal action $\xb_k^{l,*}$.
Let $\cG_K$ be the event that the high probability statement in Corollary \ref{col:beta-sup} holds.

\begin{lemma}[Formal statement of Lemma~\ref{lm:optimal1}]\label{lm:optimal}
    For any level $l > 0$, assume event $\cG_K$ holds, then there exists $\xb_k^{l,*} \in \cD_k^l$, $r(\xb^*_k) - r(\xb_k^{l,*}) \leq 2(l-1)\zeta\left(1 + 4 \sqrt{d\iota_{1}(l)}\right)$ where $\iota_{1}(l) = \log\left(3LB2^l\right)$.
\end{lemma}

\begin{proof}
    We would prove the statement by induction. Since $\cD_k^1 = \cD_k$, we have $\xb^*_k \in \cD_k^1$ and thus the induction basis holds according to $r(\xb^*_k) - r(\xb_k^{l,*}) = 0$. Now we assume the statement holds for level $l$, that is, there exists $\xb_k^{l,*} \in \cD_k^{l}$ such that $\xb_k^{l,*} \in \cD_k^l$, $r(\xb^*_k) - r(\xb_k^{l,*}) \leq  2(l-1)\zeta\left(1 + 4 \sqrt{d\iota_{1}(l)}\right)$.

    If $\xb_k^{l,*} \in \cD_k^{l+1}$, then the desired statement directly holds by choosing $\xb_k^{l,*} = \xb_k^{l-1,*}$. Otherwise $\xb_k^{l,*}$ is eliminated by some action $\xb_k^{l+1, *} \in \cD_k^{l}$ that $r_k^l(\xb_k^{l+1, *}) \ge r_k^l(\xb_k^{l, *}) + 2\beta(l) 2^{-l}$.
    Moreover, from the definition of estimator $r_k^l(\cdot)$, we have 
    \begin{align}
         r_k^l(\xb_k^{l+1, *}) - r(\xb_k^{l+1, *}) \le \zeta + \left\langle \xb_k^{l+1, *}, \theta_k^l - \theta^*\right\rangle + \beta(l) \left \|\xb_k^{l+1, *}\right\|_{(\Ub_k^l)^{-1}}  \label{eq:d4-1}
    \end{align}
    and 
    \begin{align}
        r(\xb_k^{l, *}) - r_k^l(\xb_k^{l, *}) \le\zeta -\left\langle \xb_k^{l, *}, \theta_k^l - \theta^* \right\rangle - \beta(l) \left \|\xb_k^{l, *}\right\|_{(\Ub_k^l)^{-1}}. \label{eq:d4-2}
    \end{align}
    Combining~\eqref{eq:d4-1} and~\eqref{eq:d4-2} and the fact that $r_k^l(\xb_k^{l+1, *}) \ge r_k^l(\xb_k^{l, *}) + 2\beta(l) 2^{-l}$ gives that 
    \begin{align*}
        r(\xb_k^{l, *}) - r(\xb_k^{l+1, *}) &\le -2\beta(l) 2^{-l} + 2\zeta + \left\langle \xb_k^{l+1, *} - \xb_k^{l, *}, \theta_k^l - \theta^* \right\rangle  - \beta(l) \left \|\xb_k^{l+1, *}\right\|_{(\Ub_k^l)^{-1}} + \beta(l) \left \|\xb_k^{l, *}\right\|_{(\Ub_k^l)^{-1}} \\
        &\le -2\beta(l) 2^{-l} + 2\zeta + 2^{-l}\left(\beta(l) + 4\zeta 2^l \sqrt{d\iota_{1}(l)} \right) + \beta(l) 2^{-l} \\
        &\le 2\zeta\left(1 + 4 \sqrt{d\iota_{1}(l)}\right),
    \end{align*}
    where the second inequality is suggested by  Corollary~\ref{col:beta-sup} and $\|\xb\|_{(\Ub_k^l)^{-1}} \le 2^{-l}$ for all $\xb \in \cD_k^l$.
    The desired statement can then be reached using the induction hypothesis.
\end{proof}

Then, the following lemma suggests that the performance of the actions in the decision set is guaranteed.
\begin{lemma}[Formal statement of Lemma~\ref{lm:decision-set1}]\label{lm:decision-set}
    For any level $l > 0$, assume event $\cG_K$ holds, then for any action $\xb \in \cD_k^l$, $r(\xb^*_k) - r(\xb) \leq 4\beta(l) 2^{-l} + 2l\zeta\left(1 + 4 \sqrt{d\iota_{1}(l)}\right)$ where $\iota_{1}(l) = \log\left(3LB2^l\right)$.
\end{lemma}

\begin{proof}   
    Let $\xb_k^{l,*} \in \cD_k^l$ be the optimal action given in Lemma~\ref{lm:optimal}.
    According to the elimination process, for any action $\xb \in \cD_k^l$, it holds that $r_k^l(\xb) \geq r_k^l(\xb_k^{l,*}) - 2\beta(l) 2^{-l}$.
    Moreover, from the definition of estimator $r_k^l(\cdot)$, we have 
    \begin{align*}
        r_k^l(\xb) - r(\xb) \le \zeta + \left\langle \xb, \theta_k^l - \theta^*\right\rangle + \beta(l) \left \|\xb\right\|_{(\Ub_k^l)^{-1}} 
    \end{align*}
    and 
    \begin{align*}
         r(\xb_k^{l, *}) - r_k^l(\xb_k^{l, *}) \le\zeta - \left\langle \xb_k^{l, *}, \theta_k^l - \theta^* \right\rangle - \beta(l) \left \|\xb_k^{l, *}\right\|_{(\Ub_k^l)^{-1}}.
    \end{align*}
    Combining the above three inequalities give 
    \begin{align*}
        r(\xb_k^{l, *}) - r(\xb) &\le 2\beta(l) 2^{-l} + 2\zeta + 2^{-l} + \left\langle \xb - \xb_k^{l, *}, \theta_k^l - \theta^* \right\rangle - \beta(l) \left \|\xb_k^{l, *}\right\|_{(\Ub_k^l)^{-1}} + \beta(l) \left \|\xb_k^{l-1, *}\right\|_{(\Ub_k^l)^{-1}} \\
        &\le 2\beta(l) 2^{-l} + 2\zeta + 2^{-l}\left(\beta(l) + 4\zeta 2^l \sqrt{d\iota_{1}(l)} \right) + \beta(l) 2^{-l} \\
        &\le 4\beta(l) 2^{-l} + 2\zeta\left(1 + 4 \sqrt{d\iota_{1}(l)}\right),
    \end{align*}
    where the second inequality is suggested by  Corollary~\ref{col:beta-sup} and $\|\xb\|_{(\Ub_k^l)^{-1}} \le 2^{-l}$ for all $\xb \in \cD_k^l$. The desired statement can then be reached by combining Lemma~\ref{lm:optimal}.
\end{proof}

\begin{proof}[Proof of Theorem~\ref{thm:main-sup}] 
    Consider the case that event $\cG_K$ holds.
    Let $l_{\Delta}$ be the smallest integer solution to $l_{\Delta} > \log (8\beta(l_{\Delta}) \Delta^{-1})$. 
    Note this relation ensures $4\beta(l_{\Delta}) 2^{-l_{\Delta}} < \Delta/2$.
    In case that the misspecification level is bounded by $2l_{\Delta}\zeta\left(1 + 4\sqrt{d\iota_{1}(l_{\Delta})}\right)  < \Delta/2$, it holds that $4\beta(l_{\Delta}) 2^{-l_{\Delta}} + 2l_{\Delta}\zeta\left(1 + 4\sqrt{d\iota_{1}(l_{\Delta})}\right)  < \Delta$. 
    According to Lemma~\ref{lm:decision-set},
    it satisfies that $$r(\xb^*_k) - r(\xb) \leq 4\beta(l_{\Delta}) 2^{-l_{\Delta}} + 2l_{\Delta}\zeta\left(1 + 4\sqrt{d\iota_{1}(l_{\Delta})}\right) $$ for any $\xb \in \cD_k^{l_{\Delta}}$. 
    According to the process of arm elimination, we have $\cD_k^{l} \subseteq \cD_k^{l_{\Delta}}$ for any $l \geq l_{\Delta}$. Thus, it holds that $r(\xb^*_k) - r(\xb) < \Delta$ for any $\xb \in \cD_k^{l}, l \geq l_{\Delta}$.
    Note that according to the definition of $\Delta$, we have $r(\xb_k^*) - r(\xb) > \Delta$ for all $\xb \in \cD_k^l$ that $ r(\xb_k^*) \neq r(\xb)$.
    These two statements together restrict $r(\xb_k^*) = r(\xb)$ for any $\xb \in \cD_k^l$ on every $l>l_{\Delta}$, that is, any action that remains in the decision sets on higher levels are optimal. Thus, we could decompose the total regret by 
    \begin{align*}
        \mathrm{Regret}(K) &= \sum_{l\ge 1} \sum_{k \in \cC_K^l} (r(\xb_k^*) - r(\xb)) = \sum_{l=1}^{l_{\Delta}-1} \sum_{k \in \cC_K^l} (r(\xb_k^*) - r(\xb)) \\
        &\leq \sum_{l=1}^{l_{\Delta}-1} |\cC_K^l| \cdot \left(4\beta(l) 2^{-l} + 2l\zeta\left(1 + 4 \sqrt{d\iota_1(l)}\right) \right) \\
        &\leq \sum_{l=1}^{l_{\Delta}-1} 16d4^l\iota_1(l) \cdot \left(4\beta(l) 2^{-l} +  2l\zeta\left(1 + 4 \sqrt{d\iota_1(l)}\right) \right) \\
        &\leq 64d  \sum_{l=1}^{l_{\Delta}-1} \beta(l) 2^l \iota_1(l) + 32 d\zeta \sum_{l=1}^{l_{\Delta}-1} l4^l\iota_{1}(l)\left(1 + 4 \sqrt{d\iota_1(l)}\right) \\
        &\leq 64 d \beta(l_{\Delta}) 2^{l_{\Delta}} \iota_{1}(l_{\Delta}) + 32d l_{\Delta}4^{l_{\Delta}} \iota_{1}(l_{\Delta}) \zeta \left(1 + 4 \sqrt{d\iota_1(l_{\Delta})}\right)  \\
        &\leq 512 d \beta^2(l_{\Delta}) \iota_{1}(l_{\Delta}) / \Delta + 2048d \beta^2(l_{\Delta}) \iota_{1}(l_{\Delta}) / \Delta \\
        &\leq 2560 d \beta^2(l_{\Delta}) \iota_{1}(l_{\Delta}) / \Delta
    \end{align*}
    where the second equality is given by Lemma~\ref{lm:decision-set}, the second inequality is given by Lemma~\ref{lm:set}, the third last inequality holds since $\beta(\cdot)$ and $\iota_1(\cdot)$ are monotone increase and the second inequality since $2^{l_{\Delta} - 1} \leq  8\beta(l_{\Delta}-1) \Delta^{-1} \leq  8\beta(l_{\Delta}) \Delta^{-1}$ and  $2l_{\Delta}\zeta\left(1 + 4\sqrt{d\iota_{1}(l_{\Delta})}\right)  < \Delta/2$.

\end{proof}

\section{Proof of Theorem~\ref{thm:lower}}\label{app:bandit-lb}
To begin with, we introduce the lemma providing a sparse vector set in $\RR^d$.
\begin{lemma}[Lemma 3.1,~\citealt{lattimore2020learning}]\label{lm:cross}
For any $\eps > 0$ and $d < [|\cD|]$ such that $d \ge \lceil 8\log(|\cD|)\eps^{-2}\rceil$, there exists a vector set $\cD \subset \RR^d$ such that $\|\xb\|_2 = 1$ for all $\xb \in \cD$ and $|\la \xb, \yb\ra| \le \eps$ for all $\xb, \yb \in \cD$ and $\xb \neq \yb$. 
\end{lemma}

Next, we present the Bretagnolle–Huber inequality providing the lower bound to distinguish a system.

\begin{lemma}[Bretagnolle–Huber inequality]\label{lm:bh}
    Let $P$ and $Q$ be probability measures on the same measurable space $(\Omega, \cF)$, let $\cA \in \cF$ be an arbitary event. Then 
    \begin{align*}
        P(\cA) + Q(\cA^c) \ge \frac12 \exp(-\mathrm{KL}(P, Q)).
    \end{align*}
\end{lemma}

For stochastic linear bandit problem with finite arm, we can denote $T_i(k)$ as the number of rounds the algorithm visit the $i$-th arm over total $k$ rounds. Then We have the KL-divergence decomposition lemma.
\begin{lemma}[Lemma 15.1, \citet{lattimore2020bandit}]\label{lm:lb-dc}
Let $\nu = (P_1, \cdots, P_n)$ be the reward distributions associated with one $n$-armed bandit and let $\nu' = (P'_1, \cdots, P'_n)$ be another $n$-armed bandit. Fix some algorithm $\pi$ and let $\PP_\nu = \PP_{\nu \pi}, \PP_{\nu'} = \PP_{\nu', \pi}$ be the probability measures on the canonical bandit model induced by the $k$-round interconnection of $\pi$ and $\nu$ (respectively, $\pi$ and $\nu'$). Then $\mathrm{KL}(\PP_\nu, \PP_{\nu'}) = \sum_{i=1}^n \EE_{\nu}[T_i(n)]\mathrm{KL}(P_i, P'_i)$
\end{lemma}

\begin{proof}[Proof of Theorem~\ref{thm:lower}]
The proof starts from inheriting the idea from~\citet{lattimore2020learning}. Given dimension $d$ and the number of arms $|\cD|$, setting $\eps = \sqrt{8\log(|\cD|) / (d - 1)}$, we can provide the contextual vector set $\cD$ such that 
\begin{align*}
    \|\xb\|_2 = 1, \forall \xb \in \cD, |\la \xb, \yb \ra| \le \sqrt{\frac{8\log(|\cD|)}{d - 1}}, \forall \xb, \yb \in \cD, \xb \neq \yb,
\end{align*}

For simplicity, we index the decision set as $\xb_1, \cdots, \xb_{|\cD|}$. Given the minimal sub-optimality gap $\Delta$, we provide the parameter set $\bTheta$ as follows:
\begin{align*}
    \bTheta = \left\{\btheta_{(i, j)} = \Delta \xb_i + 2\Delta \xb_j, \xb_i, \xb _j \in \cD, i \neq j\right\} \bigcup \{\btheta_{i} = \Delta \xb_i, \xb_i \in \cD\}.
\end{align*}
It can be verified that $\bTheta$ contains two kinds of $\btheta$. The first one $\btheta_{(i, j)}$ is a mixture of two different contexts $\xb_i, \xb_j$ with different strength $\Delta$ and $2\Delta$. The second one is $\btheta_{i}$ which only contains features from one context $\xb_i$. We can further verify that the size of $|\bTheta| = |\cD|^2$ and $\|\btheta\|_2 \le \sqrt{5}\Delta$ for $\btheta \in \bTheta$. For different parameter $\btheta$, the reward function is sampled from a Gaussian distribution $\cN(r_{\btheta}(\xb), 1)$, where the expected reward function is defined as
\begin{align*}
    r_{\btheta_{(i, j)}}(\xb) = \begin{cases}
        2\Delta \text{ if } \xb = \xb_j\\
        \Delta \text{ if } \xb = \xb_i\\
        0 \text{ otherwise }
    \end{cases},
    r_{\btheta_{i}}(\xb) = \begin{cases}
        \Delta \text{ if } \xb = \xb_i\\
        0 \text{ otherwise }
    \end{cases}.
\end{align*}

We can verify that the minimal sub-optimality of all these bandit problem is $\Delta$. For different parameter $\btheta$ and input $\xb$, by utilizing the sparsity of the set $\cD$ (i.e. $|\xb^\top y| \le \eps$ if $\xb \neq \yb$), we can verify the misspecification level as
\begin{align*}
    |r_{\btheta_{(i, j)}}(\xb) - \btheta_{(i, j)}^\top \xb | &= \begin{cases}
        |2\Delta - 2\Delta\xb_j^\top\xb - \Delta \xb_i^\top\xb| \le \Delta \eps \text{ if } \xb = \xb_j\\
        |\Delta - 2\Delta \xb_j^\top\xb - \Delta \xb_i^\top \xb| \le 2\Delta \epsilon \text{ if } \xb = \xb_i\\
        |0 - 2\Delta \xb_j^\top \xb - \Delta \xb_i^\top \xb| \le 3\Delta \eps \text{ otherwise}
    \end{cases}\\
    |r_{\btheta_i}(\xb) - \btheta_i^\top(\xb)| &= \begin{cases}
        |\Delta - \Delta \xb_i^\top \xb| = 0 \text{ if } \xb = \xb_i\\
        |0 - \Delta \xb_i^\top \xb| \le \Delta \eps \text{ otherwise}.
    \end{cases}
\end{align*}
Therefore we have verified that the misspecification level is bounded by $\zeta = 3\Delta \eps$. 

The provided bandit structure is hard for any linear algorithm to learn since any algorithm cannot get any information before it encounters non-zero expected rewards, even regardless of the noise of the rewards. We following the same method in~\citet{lattimore2020bandit}. If the algorithm choose arm $i$ at the first round, there would be $|\cD|$ parameters (i.e. $\btheta_i, \btheta_{(i, \cdot)}$ receiving a non-zero expected reward. On the second round if the algorithm choose a different arm $j$, there would be $|\cD|$ parameters (i.e. $\btheta_j, \btheta_{(j, k: k \neq i)}$ receiving a non-zero expected reward. Therefore the average time of receiving zero expected reward should be 
\begin{align*}
    |\cD|^{-2}\sum_{i=1}^{|\cD|}(i - 1)(|\cD| - i + 1) &= |\cD|^{-2}\sum_{i=0}^{|\cD| - 1}i(|\cD| - i)\\
    &= |\cD|^{-2}\left(|\cD|\sum_{i=0}^{|\cD| - 1} i - \sum_{i=0}^{|\cD| - 1} i^2\right)\\
    &= |\cD|^{-2}\left(\frac{|\cD|^2(|\cD| - 1)}{2} - \frac{|\cD|(|\cD| - 1)(2|\cD| - 1)}{6}\right)\\
    &= \frac{|\cD| - 1}{2}\left(1 - \frac{2|\cD| - 1}{3|\cD|}\right)\\
    &\ge \frac{|\cD| - 1}{6},
\end{align*}
where the third equation is from the fact that $\sum_{i=1}^n i = n(n + 1) / 2$ and $\sum_{i=1}^n i^2 = n(n+1)(2n + 1) / 6$. The last inequality is from the fact that $2|\cD| - 1) / (3|\cD|) \le 2/3$. Therefore, even without of the random noise, any algorithm is expected to receive $\min\{K, (|\cD| - 1) / 6\}$ uninformative data with expected reward to be zero. Therefore any algorithm will receive a $\Delta\min\{K, (|\cD| - 1) / 6\}$ regret considers the suboptimality as $\Delta$.

Next, we consider the effect of random noise. For any algorithm running on this parameter set $\bTheta$, we find two parameter $\btheta_{i}$ and $\btheta_{i, j}$ where $j \neq i$. Define the event as $\cA = \{T_j(k) \ge k / 2\}$ and $\cA^c = \{T_j(k) < k / 2\}$. By Lemma~\ref{lm:bh} and Lemma~\ref{lm:lb-dc},
\begin{align}
    \PP_{\btheta_i}\left(T_j(k) \ge \frac k 2\right) + \PP_{\btheta_{(i, j)}}\left(T_j(k) < \frac k 2\right) &\ge \frac12 \exp(-\mathrm{KL}(\PP_{\btheta_i}, \PP_{\btheta_{(i, j)}})) \notag \\
    &\ge \frac12 \exp\left(-\sum_{n \in \cD} \EE_{\btheta_i}[T_n(k)]\mathrm{KL}\left(\PP_{\btheta_{(i, j)}, n}, \PP_{\btheta_j, n}\right)\right). \label{eq:new:2}
\end{align}

Noticing the minimal sub-optimality gap is $\Delta$. Also the $j$-th arm is the sub-optimal arm for parameter $\btheta_i$. Therefore, once $T_j(k) \ge k / 2$, the algorithm will at least suffer from $\Delta k / 2$ regret for parameter $\btheta_i$. Also, since the $j$-th arm is the optimal arm for bandit $\btheta_{(i, j)}$. If $T_j(k) < k / 2$, the algorithm will also at least suffer from $\Delta k / 2$ regret for $\btheta_{(i, j)}$. Denoting $\cR_{\btheta}(k)$ as the expected cumulative regret over $k$ rounds, that is to say
\begin{align}
    \cR_{\btheta_i}(k) \ge \frac{\Delta k}2\PP_{\btheta_i}(T_j(k) \ge k / 2)\quad
    \cR_{\btheta_j}(k) \ge \frac{\Delta k}2\PP_{\btheta_i}(T_j(k) < k / 2).\label{eq:new:1}
\end{align}

On the other hand since the bandit using $\btheta_i$ and $\btheta_j$ only differ in the $j$-th arm. Since standard Gaussian noise is adapted, $\mathrm{KL}( \PP_{\btheta_i, n}, \PP_{\btheta_{(i, j)}, n}) = \Delta^2\ind[n = j] / 2$. Combining this with \eqref{eq:new:1}, \eqref{eq:new:2} suggests that
\begin{align*}
    \cR_{\btheta_i}(k) + \cR_{\btheta_j}(k) \ge \frac{\Delta k}2 \exp\left(-\frac{\Delta^2}2\EE_{\btheta_i} \left[T_j(k)\right]\right),
\end{align*}
which suggests that 
\begin{align}
    \EE_{\btheta_i} \left[T_j(k)\right] \ge \frac{\log(\Delta k) - \log 2 - \log(\cR_{\btheta_i}(k) + \cR_{\btheta_j}(k))}{\Delta^2 / 2}, \label{eq:neq1}
\end{align}
For any algorithm seeking to get a sublinear expected regret bound of $\cR_{\btheta}(k) \le Ck^\alpha$ with $C > 0, 0 \le \alpha < 1$ for all $\btheta \in \bTheta$, \eqref{eq:neq1} becomes
\begin{align}
    \EE_{\btheta_i} \left[T_j(k)\right] \ge \frac{\log(\Delta k) - \log 2 - \log(2Ck^\alpha)}{\Delta^2 / 2} = \frac{\log(\Delta k) - \log(4C) - \alpha \log k}{\Delta^2 / 2}. \label{eq:neq2}
\end{align}

Since that the regret on $\btheta_i$ can be decomposed by
\begin{align}
    \cR_{\btheta_i}(k) = \Delta \sum_{n = 1, n \neq i}^{|\cD|} T_n(k), \label{eq:neq3}
\end{align}
combining~\eqref{eq:neq3} with~\eqref{eq:neq2} yields
\begin{align*}
    \cR_{\btheta_i}(k) \ge \frac{2(|\cD| - 1)}{\Delta}\max\left\{\log (\Delta k) - \log(4C) - \alpha \log k, 0\right\},
\end{align*}
where the $\max$ operator is trivially taken for $\cR_{\btheta}(k) \ge 0$.
\end{proof}

\bibliography{refs.bib}

\begin{thebibliography}{26}
\expandafter\ifx\csname natexlab\endcsname\relax\def\natexlab#1{#1}\fi
\expandafter\ifx\csname url\endcsname\relax
  \def\url#1{\texttt{#1}}\fi
\expandafter\ifx\csname urlprefix\endcsname\relax\def\urlprefix{URL }\fi

\bibitem[{Abbasi-Yadkori et~al.(2011)Abbasi-Yadkori, P{\'a}l and
  Szepesv{\'a}ri}]{abbasi2011improved}
\textsc{Abbasi-Yadkori, Y.}, \textsc{P{\'a}l, D.} and \textsc{Szepesv{\'a}ri,
  C.} (2011).
\newblock Improved algorithms for linear stochastic bandits.
\newblock \textit{Advances in neural information processing systems}
  \textbf{24} 2312--2320.

\bibitem[{Agrawal and Goyal(2013)}]{agrawal2013thompson}
\textsc{Agrawal, S.} and \textsc{Goyal, N.} (2013).
\newblock Thompson sampling for contextual bandits with linear payoffs.
\newblock In \textit{International Conference on Machine Learning}. PMLR.

\bibitem[{Auer(2002)}]{auer2002using}
\textsc{Auer, P.} (2002).
\newblock Using confidence bounds for exploitation-exploration trade-offs.
\newblock \textit{Journal of Machine Learning Research} \textbf{3} 397--422.

\bibitem[{Camilleri et~al.(2021)Camilleri, Jamieson and
  Katz-Samuels}]{camilleri2021high}
\textsc{Camilleri, R.}, \textsc{Jamieson, K.} and \textsc{Katz-Samuels, J.}
  (2021).
\newblock High-dimensional experimental design and kernel bandits.
\newblock In \textit{International Conference on Machine Learning}. PMLR.

\bibitem[{Chu et~al.(2011)Chu, Li, Reyzin and Schapire}]{chu2011contextual}
\textsc{Chu, W.}, \textsc{Li, L.}, \textsc{Reyzin, L.} and \textsc{Schapire,
  R.} (2011).
\newblock Contextual bandits with linear payoff functions.
\newblock In \textit{Proceedings of the Fourteenth International Conference on
  Artificial Intelligence and Statistics}. JMLR Workshop and Conference
  Proceedings.

\bibitem[{Deng et~al.(2009)Deng, Dong, Socher, Li, Li and Fei-Fei}]{5206848}
\textsc{Deng, J.}, \textsc{Dong, W.}, \textsc{Socher, R.}, \textsc{Li, L.-J.},
  \textsc{Li, K.} and \textsc{Fei-Fei, L.} (2009).
\newblock Imagenet: A large-scale hierarchical image database.
\newblock In \textit{2009 IEEE Conference on Computer Vision and Pattern
  Recognition}.

\bibitem[{Du et~al.(2019)Du, Kakade, Wang and Yang}]{du2019good}
\textsc{Du, S.~S.}, \textsc{Kakade, S.~M.}, \textsc{Wang, R.} and \textsc{Yang,
  L.~F.} (2019).
\newblock Is a good representation sufficient for sample efficient
  reinforcement learning?
\newblock In \textit{International Conference on Learning Representations}.

\bibitem[{Elson et~al.(2007)Elson, Douceur, Howell and
  Saul}]{asirra-a-captcha-that-exploits-interest-aligned-manual-image-categorization}
\textsc{Elson, J.}, \textsc{Douceur, J.~J.}, \textsc{Howell, J.} and
  \textsc{Saul, J.} (2007).
\newblock Asirra: A captcha that exploits interest-aligned manual image
  categorization.
\newblock In \textit{Proceedings of 14th ACM Conference on Computer and
  Communications Security (CCS)}. Association for Computing Machinery, Inc.

\bibitem[{Foster et~al.(2020)Foster, Gentile, Mohri and
  Zimmert}]{foster2020adapting}
\textsc{Foster, D.~J.}, \textsc{Gentile, C.}, \textsc{Mohri, M.} and
  \textsc{Zimmert, J.} (2020).
\newblock Adapting to misspecification in contextual bandits.
\newblock \textit{Advances in Neural Information Processing Systems}
  \textbf{33}.

\bibitem[{Ghosh et~al.(2017)Ghosh, Chowdhury and
  Gopalan}]{ghosh2017misspecified}
\textsc{Ghosh, A.}, \textsc{Chowdhury, S.~R.} and \textsc{Gopalan, A.} (2017).
\newblock Misspecified linear bandits.
\newblock In \textit{Proceedings of the AAAI Conference on Artificial
  Intelligence}, vol.~31.

\bibitem[{Hao et~al.(2020)Hao, Lattimore and Szepesvari}]{hao2020adaptive}
\textsc{Hao, B.}, \textsc{Lattimore, T.} and \textsc{Szepesvari, C.} (2020).
\newblock Adaptive exploration in linear contextual bandit.
\newblock In \textit{International Conference on Artificial Intelligence and
  Statistics}. PMLR.

\bibitem[{He et~al.(2021{\natexlab{a}})He, Zhou and Gu}]{he2021logarithmic}
\textsc{He, J.}, \textsc{Zhou, D.} and \textsc{Gu, Q.} (2021{\natexlab{a}}).
\newblock Logarithmic regret for reinforcement learning with linear function
  approximation.
\newblock In \textit{International Conference on Machine Learning}. PMLR.

\bibitem[{He et~al.(2021{\natexlab{b}})He, Zhou and Gu}]{he2021uniformpac}
\textsc{He, J.}, \textsc{Zhou, D.} and \textsc{Gu, Q.} (2021{\natexlab{b}}).
\newblock Uniform-{PAC} bounds for reinforcement learning with linear function
  approximation.
\newblock In \textit{Advances in Neural Information Processing Systems}.

\bibitem[{He et~al.(2022)He, Zhou, Zhang and Gu}]{he22corruptions}
\textsc{He, J.}, \textsc{Zhou, D.}, \textsc{Zhang, T.} and \textsc{Gu, Q.}
  (2022).
\newblock Nearly optimal algorithms for linear contextual bandits with
  adversarial corruptions.
\newblock In \textit{Advances in Neural Information Processing Systems}.

\bibitem[{He et~al.(2016)He, Zhang, Ren and Sun}]{he2016deep}
\textsc{He, K.}, \textsc{Zhang, X.}, \textsc{Ren, S.} and \textsc{Sun, J.}
  (2016).
\newblock Deep residual learning for image recognition.
\newblock In \textit{Proceedings of the IEEE conference on computer vision and
  pattern recognition}.

\bibitem[{Lattimore and Szepesv{\'a}ri(2020)}]{lattimore2020bandit}
\textsc{Lattimore, T.} and \textsc{Szepesv{\'a}ri, C.} (2020).
\newblock \textit{Bandit algorithms}.
\newblock Cambridge University Press.

\bibitem[{Lattimore et~al.(2020)Lattimore, Szepesvari and
  Weisz}]{lattimore2020learning}
\textsc{Lattimore, T.}, \textsc{Szepesvari, C.} and \textsc{Weisz, G.} (2020).
\newblock Learning with good feature representations in bandits and in rl with
  a generative model.
\newblock In \textit{International Conference on Machine Learning}. PMLR.

\bibitem[{Li et~al.(2010)Li, Chu, Langford and Schapire}]{li2010contextual}
\textsc{Li, L.}, \textsc{Chu, W.}, \textsc{Langford, J.} and \textsc{Schapire,
  R.~E.} (2010).
\newblock A contextual-bandit approach to personalized news article
  recommendation.
\newblock In \textit{Proceedings of the 19th international conference on World
  wide web}.

\bibitem[{Mnih et~al.(2013)Mnih, Kavukcuoglu, Silver, Graves, Antonoglou,
  Wierstra and Riedmiller}]{mnih2013playing}
\textsc{Mnih, V.}, \textsc{Kavukcuoglu, K.}, \textsc{Silver, D.},
  \textsc{Graves, A.}, \textsc{Antonoglou, I.}, \textsc{Wierstra, D.} and
  \textsc{Riedmiller, M.} (2013).
\newblock Playing atari with deep reinforcement learning.
\newblock \textit{arXiv preprint arXiv:1312.5602} .

\bibitem[{Papini et~al.(2021)Papini, Tirinzoni, Restelli, Lazaric and
  Pirotta}]{papini2021leveraging}
\textsc{Papini, M.}, \textsc{Tirinzoni, A.}, \textsc{Restelli, M.},
  \textsc{Lazaric, A.} and \textsc{Pirotta, M.} (2021).
\newblock Leveraging good representations in linear contextual bandits.
\newblock In \textit{International Conference on Machine Learning}. PMLR.

\bibitem[{Schulman et~al.(2015)Schulman, Levine, Abbeel, Jordan and
  Moritz}]{schulman2015trust}
\textsc{Schulman, J.}, \textsc{Levine, S.}, \textsc{Abbeel, P.},
  \textsc{Jordan, M.} and \textsc{Moritz, P.} (2015).
\newblock Trust region policy optimization.
\newblock In \textit{International conference on machine learning}. PMLR.

\bibitem[{Schulman et~al.(2017)Schulman, Wolski, Dhariwal, Radford and
  Klimov}]{schulman2017proximal}
\textsc{Schulman, J.}, \textsc{Wolski, F.}, \textsc{Dhariwal, P.},
  \textsc{Radford, A.} and \textsc{Klimov, O.} (2017).
\newblock Proximal policy optimization algorithms.
\newblock \textit{arXiv preprint arXiv:1707.06347} .

\bibitem[{Shalev-Shwartz and Ben-David(2014)}]{shalev2014understanding}
\textsc{Shalev-Shwartz, S.} and \textsc{Ben-David, S.} (2014).
\newblock \textit{Understanding machine learning: From theory to algorithms}.
\newblock Cambridge university press.

\bibitem[{Takemura et~al.(2021)Takemura, Ito, Hatano, Sumita, Fukunaga,
  Kakimura and Kawarabayashi}]{takemura2021parameter}
\textsc{Takemura, K.}, \textsc{Ito, S.}, \textsc{Hatano, D.}, \textsc{Sumita,
  H.}, \textsc{Fukunaga, T.}, \textsc{Kakimura, N.} and \textsc{Kawarabayashi,
  K.-i.} (2021).
\newblock A parameter-free algorithm for misspecified linear contextual
  bandits.
\newblock In \textit{International Conference on Artificial Intelligence and
  Statistics}. PMLR.

\bibitem[{Van~Roy and Dong(2019)}]{van2019comments}
\textsc{Van~Roy, B.} and \textsc{Dong, S.} (2019).
\newblock Comments on the du-kakade-wang-yang lower bounds.
\newblock \textit{arXiv preprint arXiv:1911.07910} .

\bibitem[{Zanette et~al.(2020)Zanette, Lazaric, Kochenderfer and
  Brunskill}]{Zanette2020LearningNO}
\textsc{Zanette, A.}, \textsc{Lazaric, A.}, \textsc{Kochenderfer, M.~J.} and
  \textsc{Brunskill, E.} (2020).
\newblock Learning near optimal policies with low inherent bellman error.
\newblock In \textit{ICML}.

\end{thebibliography}
\bibliographystyle{ims}

\end{document}